%% file: submission.tex
\documentclass[lettersize,journal]{IEEEtran}

\usepackage{amsthm}
\usepackage{amssymb}
\usepackage{enumerate}
\usepackage{algorithm}
\usepackage{algpseudocode}
\usepackage{graphicx}
\graphicspath{{./figure/}}
\usepackage{multirow}
\usepackage{amsmath}
\usepackage{ulem}

\newtheorem{theorem}{Theorem}

\newtheorem{definition}{Definition}

\def \ourtool{JANUS\ }

\begin{document}

\title{\ourtool: A Difference-Oriented Analyzer For Financial Centralized Risks in Smart Contracts}

\author{Wansen Wang, Pu Zhang, Renjie Ji, Wenchao Huang, Zhaoyi Meng, Jie Cui, Hong Zhong, Yan Xiong \thanks{Wansen Wang, Pu Zhang, Zhaoyi Meng , Jie Cui and Hong Zhong are with the School of Computer Science and Technology, Anhui University. Email: 22112@ahu.edu.cn}\thanks{Renjie Ji, Wenchao Huang, Yan Xiong  are with the School of Computer Science and Technology, University of Science and Technology of China. Email:huangwc@ustc.edu.cn}\thanks{(Corresponding authors: Zhaoyi Meng, Wenchao Huang)}}

\maketitle

\begin{abstract}

Some smart contracts violate decentralization principles by defining privileged accounts that manage other users' assets without permission, introducing centralized risks that have caused financial losses. 
Existing methods, however, face challenges in accurately detecting diverse centralized risks due to their dependence on predefined behavior patterns.
In this paper, we propose JANUS, an automated analyzer for Solidity smart contracts that detects financial centralized risks independently of their specific behaviors.
JANUS identifies differences between states reached by privileged and ordinary accounts, and analyzes whether these differences are finance-related. 
Focusing on the impact of risks rather than behaviors, JANUS achieves improved accuracy compared to existing tools and can uncover centralized risks with unknown patterns.

To evaluate JANUS's performance, we compare it with other tools using a dataset of 540 contracts. 
Our evaluation demonstrates that JANUS outperforms representative tools in terms of detection accuracy for financial centralized risks. 
Additionally, we evaluate JANUS on a real-world dataset of 33,151 contracts, successfully identifying two types of risks that other tools fail to detect.
We also prove that the state traversal method and variable summaries, which are used in JANUS to reduce the number of states to be compared, do not introduce false alarms or omissions in detection. 

\end{abstract}


\input{introduction}

\input{background}

\input{backdoors}

\input{design}

\input{design2}
\input{discuss}

\input{experiment}

\input{related}
\input{limitation}
\input{conclusion}

\normalem
\bibliographystyle{IEEEtran}
\bibliography{submission}

% \clearpage
% \setcounter{page}{1}

% \input{changes2}

\appendix

\input{appendix}

% that's all folks
\end{document}

%% file: introduction.tex
%!TEX root = submission.tex

\section{Introduction}
\label{sec:intro}

\IEEEPARstart{S}{olidity}, one of the most popular smart contract languages~\cite{hwang2020gap}, is supported by several blockchain platforms, e.g., Ethereum \cite{ethereum} and Binance Smart Chain \cite{bsc}. 
Solidity smart contracts have been widely adopted, managing digital assets valued at over \$10 billion \cite{marketcap}. 
% However, this widespread adoption and significant economic value have unfortunately attracted malicious attackers, including those developing backdoors.
% A contract backdoor is a segment of malicious code embedded within a contract that exclusively permits attackers to invoke it, allowing them to manipulate users' assets or restrict users' access to functions.
% %A smart contract backdoor is a malicious code planted in a contract with specific permissions, allowing attackers to steal assets, disable user functions, and more. 
% %The emergence of smart contract backdoors poses a serious threat, potentially leading to financial losses for numerous contract users.
% Both Ethereum and Binance Smart Chain have experienced security breaches due to backdoors. 
% For example, in June 2018, an Australian company lost \$6.6 million due to a malicious backdoor in the SoarCoin smart contract on Ethereum \cite{attack1}. 
% In November 2021, the Squid token developer prevented users from reselling their assets through a backdoor, resulting in a \$3.38 million loss \cite{attack2}. 
% Given the financial impact of smart contract backdoors, it is necessary to propose methods for their detection and risk mitigation.
However, some smart contracts introduce the risk of centralization in their attempt to facilitate the management of users (or claim to do so).
For instance, the SoarCoin contract included a special function allowing the token issuer to retrieve tokens at no cost. 
Although developers claimed this feature was intended for airdrops and development activities, its misuse in June 2018 resulted in an Australian company losing assets worth \$6.6 million \cite{attack1}. 
A similar incident occurred with the Squid token in November 2021, causing financial losses of \$3.38 million for contract users \cite{attack2}. 
These cases highlight a common issue: contracts defining privileged accounts that can manage other users' assets without permission, thereby introducing centralized risks. 
Such risks not only pose potential threats to contract users but can also lead to financial losses. 
Consequently, there is an urgent need for methods to detect and mitigate contract centralized risks to protect users' assets.

% Researchers have proposed automated detection methods for centralized risks in smart contracts. 
% Pied-piper \cite{ma2023pied} summarizes five common patterns of code with centralized risks and uses static Datalog analysis \cite{immerman1998descriptive} to identify contracts conforming to these patterns, supplemented by directed fuzzy testing \cite{bohme2017directed} to reduce false positives. 
% Similarly, Tokeer \cite{zhou2024stop} analyzes transfer-related functional modules in contract code, generates oracles based on four known rug pull contract patterns, and detects rug pull contracts using Datalog analysis. 
% The rug pull contracts \cite{sun2024sok}, a type of malicious fraud contract, often incorporate code with centralized risks for asset transfer, making detection methods for this category applicable to centralized risks as well.

Researchers have proposed automated detection methods for centralized risks in smart contracts. 
Pied-piper \cite{ma2023pied} identifies five common patterns of code associated with centralized risks. 
It employs static Datalog analysis \cite{immerman1998descriptive} to detect contracts conforming to these patterns, followed by directed fuzzy testing \cite{bohme2017directed} to minimize false positives. 
Similarly, Tokeer \cite{zhou2024stop} focuses on transfer-related functional modules in contract code. 
It generates oracles based on known rug pull contract patterns and uses Datalog analysis for detection. 
Rug pull contracts~\cite{sun2024sok}, a type of  fraud contracts, often incorporate code for asset transfer with centralized risks. Consequently, detection methods developed for rug pull contracts can also be applied to identify centralized risks in smart contracts.

% Researchers have already proposed some automatic detection methods for smart contract backdoors.
% For instance, Pied-piper\cite{ma2023pied} summarizes 5 common code patterns of smart contract backdoors and employs Datalog analysis to determine if a contract conforms to a given pattern, supplemented by directed fuzzy testing to reduce false positives.
% Similarly, Tokeer\cite{zhou2024stop} builds a transfer model to analyze the functional modules in the contract code implementing transfers, generating oracles based on 4 known rug pull contract patterns, and ultimately detects the rug pull contract based on the oracles through Datalog analysis.
% The rug pull contracts here are a kind of malicious fraud contracts, which often necessitate the insertion of backdoors for asset transfer, rendering detection methods applicable to this category also capable of uncovering smart contract backdoors.

However, existing methods face challenges in accurately detecting diverse code with centralized risks in real-world scenarios due to their reliance on predefined behavior patterns.
%However, existing methods struggle to accurately detect diverse backdoors in real-world scenarios due to their reliance on predefined behavior patterns. 
They may underreport when faced with unknown patterns or variants of known behavior patterns. 
Additionally, inaccurate predefined patterns can lead to misclassification of secure contracts as risky ones.

% However, due to the reliance on predefined attack patterns, it is difficult for existing methods to accurately detect diverse backdoors in real-world scenarios.
% On the one hand, when facing backdoors with unknown patterns or variants of known behavioral patterns, existing methods suffer from underreporting.
% On the other hand, an inaccurate predefined pattern may also cause misclassifying secure contracts as contracts with backdoors.

% However, existing methods relying on predefined patterns need manual intervention for pattern modification or supplementation when facing backdoors with unknown patterns or variants of known backdoors, or else underreporting occurs, which makes the current methods unsuitable for real-world scenarios with numerous contracts and diverse backdoor attack patterns\cite{event1}\cite{event2}\cite{event3}.

\begin{figure}[H]
\centering
\includegraphics[scale=0.33]{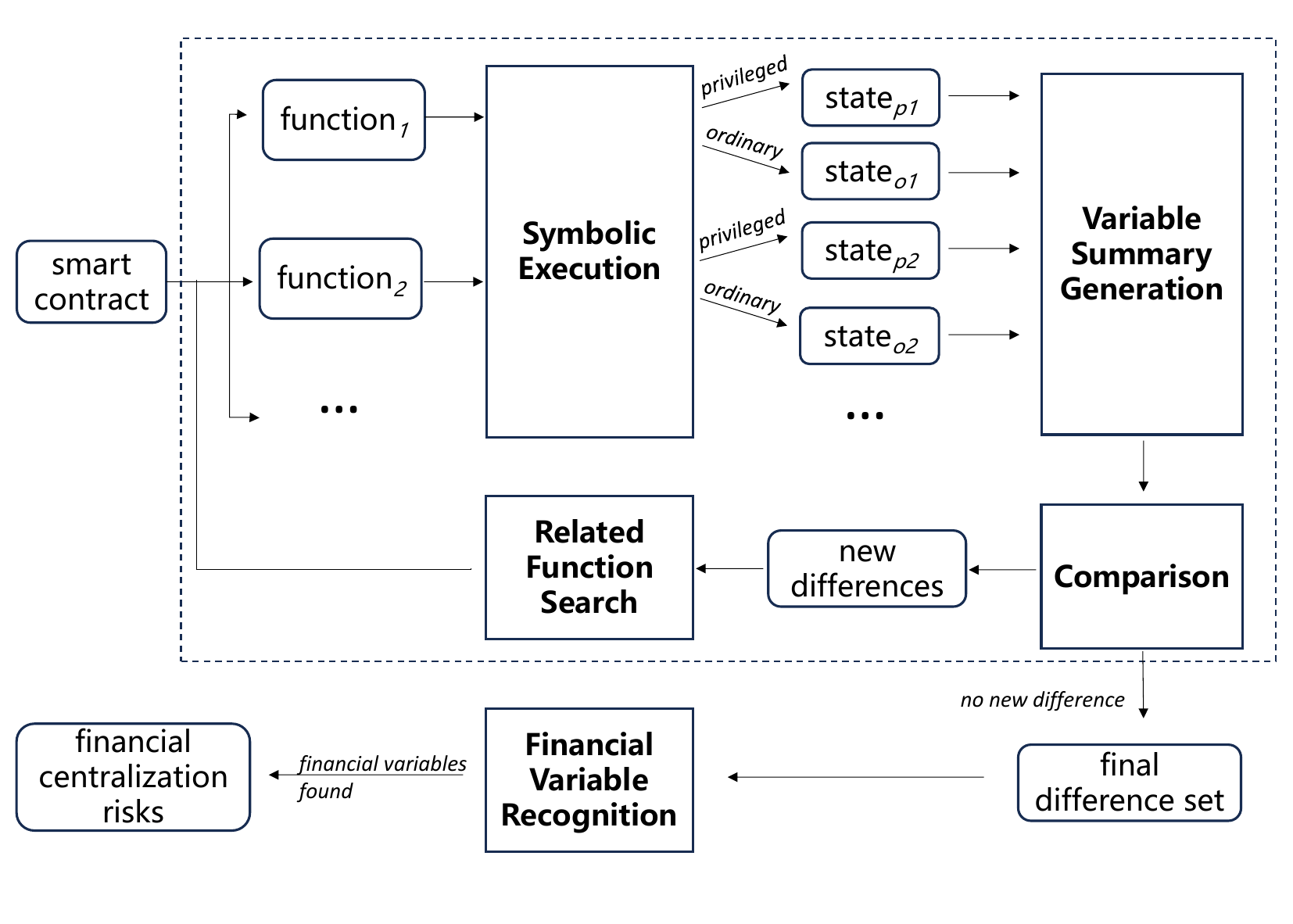}
\caption{Overview of our framework.}
\label{fig:overview}
\end{figure}

In this paper, we propose a detection framework for centralized risks in smart contracts independently of specific behavior patterns, offering improved accuracy compared to existing methods and the ability to uncover risky codes of unknown patterns.

First, our observation reveals that existing risky code often leverages privileged accounts to manipulate variables through execution paths inaccessible to ordinary accounts. 
Privileged accounts are those designated by developers for management purposes, such as the contract owner.
This paper aims to detect centralized risks associated with these privileged accounts to avoid financial losses for users.
Therefore, we focus on manipulation made by these accounts to financial variables used to represent digital asset balances, e.g., the commonly used variable $\texttt{balanceOf}$ indicating token balances.

Based on this observation, we propose an automated detection framework for centralized risks, as illustrated in Fig. \ref{fig:overview}. 
The framework obtains differences between states reached by privileged and ordinary accounts through an iterative algorithm (shown in the dotted box) and analyzes whether these differences relate to financial variables. 
% The framework obtains a difference set of function execution results through an iterative method (shown in the dotted box) and analyzes whether these differences relate to financial variables. 
If so, it indicates that privileged accounts can perform specialized operations on their own or ordinary accounts' financial assets, suggesting the possible presence of risks. By focusing on the ultimate impact of risks rather than specific behavior patterns, our framework enables more accurate detection than existing tools.

% Based on the above observation, we propose an automatic detection framework for smart contract backdoors, the workflow of which is shown in Fig. \ref{fig:overview}.
% %We synthesize the above methods to implement an automatic detection framework for smart contract backdoors, the workflow of which is shown in Fig. \ref{fig:overview}.
% The framework obtains the difference set of function execution results through an iterative method (as shown in the dotted box), and then analyzes whether the differences are related to the financial variables.
% If so,  it indicates that privileged accounts can execute specialized operations on their own or ordinary accounts' financial assets, signifying the possible presence of backdoors.
% Our framework focuses on the ultimate impact of the backdoors and is not limited to specific attack patterns, thus enabling accurate backdoor detection.

Specifically, the iterative algorithm for obtaining differences consists of two main steps:
\begin{enumerate}[1)]
\item Given a contract, symbolically execute the functions within it using both privileged and ordinary accounts, then collect the differences in outcomes into an initial difference set.
\item Identify functions affected by the differences in the set. Execute these functions using accounts with different permissions and compare the resultant states. If new differences emerge, add them to the difference set and repeat this step. Terminate the process when no new differences are found.
\end{enumerate}
This algorithm aims to converge on a set of differences by continually propagating and analyzing differences in states reached by accounts with different permissions. 
Additionally, we introduce a method based on graph neural networks (GNN)~\cite{wu2022graph} to automatically determine if the final differences contain financial variables.

However, the iterative algorithm faces two challenges:
% \begin{enumerate}[1)]
\begin{itemize}
\item The number of states to be traversed and compared grows exponentially with the number of function executions when each function is executed by two different accounts (privileged and ordinary accounts).
\item Contract functions can be executed an unlimited number of times, potentially generating an infinite number of result differences. This may lead to new differences in each iteration of the iterative algorithm, preventing the algorithm from terminating.
\end{itemize}
%\end{enumerate}
% However, the above method faces two challenges:
% 1. For each function executed with two different identities, privileged accounts and ordinary accounts, the number of states to be traversed grows exponentially with the number of times the function is executed.
% 2. The functions in contracts can be executed any number of times, and thus an infinite number of result differences can be generated, possibly leading to new difference in each round of Step 2, and the above approach cannot be terminated.
% To address Challenge 1, we propose state traversal of xxx, which updates the current state only with the execution result of the privileged account, and the number of states to be traversed grows linearly with the number of function executions.
% We prove that under our assumptions, xxx's state traversal does not lead to underreporting of backdoors.
% To address Challenge 2, we abstract the values of the variables in the execution results as variable summaries, and the finite nature of the variable summaries we define ensures that the method ultimately converges to a set of summarized differences.
% We prove that under our assumptions, variable summaries do not lead to false alarms or omissions of backdoors.
% Moreover, we introduce a graph neural network-based method to automatically identify financial variables within smart contracts.

To address the first challenge, we introduce a difference-oriented state traversal method. 
This method focuses on a subset of the contract state space and adds labels for variables exhibiting differences during comparison to preserve the information of states that are omitted.
%uses the states reached by privileged accounts as the initial states of the next iteration, disregarding the subsequent states resulting from ordinary account executions. 
%This method updates the current state solely based on the execution results of privileged accounts, disregarding the subsequent states resulting from ordinary account executions. 
By doing so, we reduce the growth of traversed states to a linear relationship with the number of function executions. 
We provide proofs that this state traversal method does not introduce false alarms or cause omissions in risk detection.
For the second challenge, we implement an abstraction method that represents the values of variables in the contract states as variable summaries. 
The finite nature of these summaries guarantees that our method ultimately converges to a set of summarized differences, addressing the potential issue of infinite iterations. Furthermore, we also prove that these variable summaries do not lead to underreporting of risks.
%do not introduce false alarms or cause omissions in backdoor detection.
% further enhancing the framework's capabilities.

Finally, we implement an automated detection tool JANUS~\cite{janus} for financial centralized risks based on our framework and evaluate its effectiveness. 
We collect a comparison dataset containing 540 smart contracts, and the experimental results on this dataset show that \ourtool is more accurate than existing tools, Pied-piper and Tokeer. 
In addition, we evaluate \ourtool in 33,151 real-world contracts, finding two kinds of risks that cannot be detected by existing tools.

Overall, the contributions of this paper are as follows:
\begin{itemize}
\item We propose a framework for detecting financial centralized risks based on the differences in contract states reached by privileged and ordinary accounts, which does not depend on the specific behavior patterns of risks. 
\item We propose an iterative algorithm for obtaining function result differences, using a difference-oriented state traversal method to optimize contract state space exploration and using variable summaries to ensure algorithm termination. We provide proofs that these methods do not introduce false alarms or omissions in risk detection.
%compromise the accuracy of backdoor detection.
\item  We propose a GNN-based method to automatically identify financial variables in smart contracts.
\item  We implement a tool \ourtool based on our framework and show that the accuracy of \ourtool is higher than that of existing tools Pied-piper and Tokeer on a dataset containing 540 smart contracts. 
\item  We evaluate \ourtool on 33,151 real-world smart contracts. Manual inspection of 1,000 contracts confirms that it achieves over 90\% precision and recall, while identifying two previously undetectable risk patterns missed by existing tools.
% We also collect 33,151 real-world smart contracts and use \ourtool to find 8,391 smart contracts with centralized risks, and find two kinds of risks with unknown patterns that cannot be detected by existing tools.
\end{itemize}

%% file: background.tex
%!TEX root = submission.tex

\section{Background}

\subsection{Solidity Smart Contracts}

\begin{figure}[h]
\centering
\includegraphics[scale=0.36]{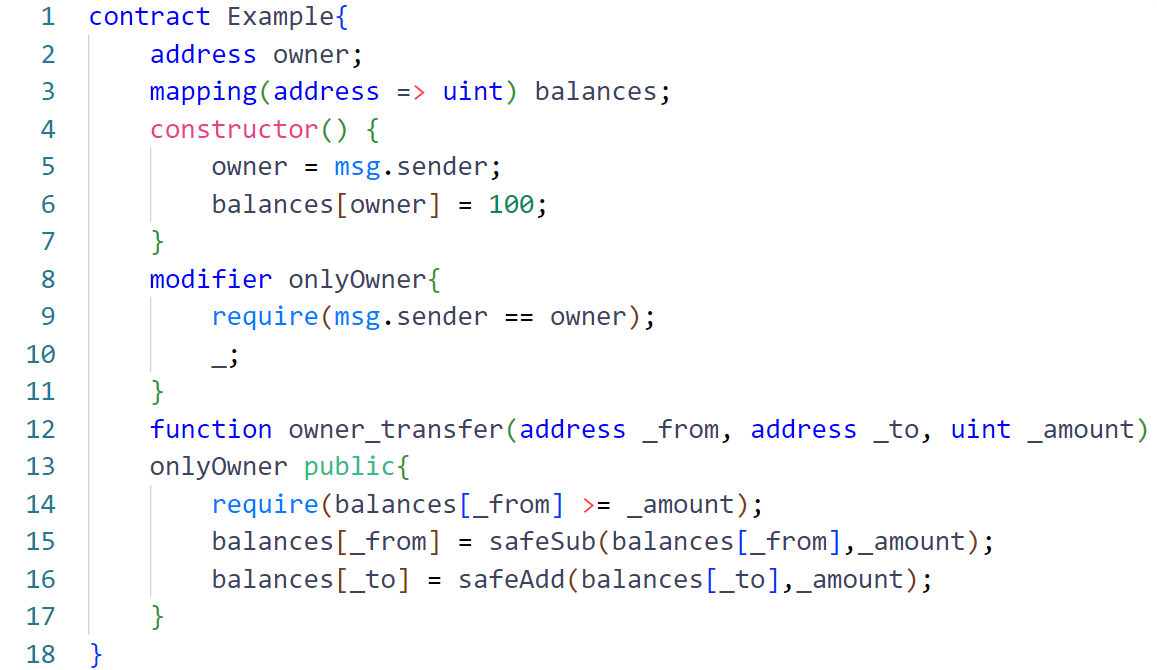}
\caption{An example of Solidity smart contracts.}
\label{fig:backdoor_transfer}
\end{figure}

% A smart contract is an automatically executing program on the blockchain, with Solidity being the most popular language for writing these contracts. 
% This paper focuses on two blockchain platforms that support Solidity: Ethereum and Binance Smart Chain (BSC), both of which are peer-to-peer networks composed of accounts. 
% These accounts are categorized into contract accounts, which are associated with and controlled by smart contracts, and external accounts, which are without code. 
% Each account is uniquely identified by an address.

Smart contracts are automatically executing programs on the blockchain, with Solidity \cite{solidity} being the most popular language for their development. 
This paper focuses on two blockchain platforms that support Solidity: Ethereum and Binance Smart Chain (BSC). 
Both platforms are peer-to-peer networks composed of two types of accounts: contract accounts, which are associated with and controlled by smart contracts, and external accounts, which contain no code. Each account is uniquely identified by an address.

A Solidity smart contract consists of state variables and functions. 
For example, in the \texttt{Example} contract shown in Fig. \ref{fig:backdoor_transfer}, Lines 2 and 3 define the state variables, while Lines 4 through 17 define the functions. State variables, persistently stored on the blockchain, can be read and modified by different functions. 
In this example, \texttt{owner} represents the account address of the contract owner, while \texttt{balances} represents the token balances of different accounts.
Functions in the contract can be invoked by accounts to execute predefined logic. 
The \texttt{constructor} in the \texttt{Example} contract is called upon contract deployment to initialize state variables.
The \texttt{owner\_transfer} function transfers tokens from \texttt{\_from} to \texttt{\_to}. 
This function is restricted by the modifier \texttt{onlyOwner}, which ensures that each call checks if the caller (\texttt{msg.sender}) is equal to the \texttt{owner}. 
If this condition is met, the function proceeds; otherwise, the call fails.

In this paper, we propose two criteria for identifying variables that represent privileged accounts (e.g., \texttt{owner} in Fig.~\ref{fig:backdoor_transfer}): 1) they are address-type state variables of contracts, and 2) their values can only be specified by the developer or other privileged accounts.

%% file: backdoors.tex
\subsection{Existing Centralized Risks in Contracts}
\label{subsec:risks}

The centralized risks, as presented in  \cite{ma2023pied} and \cite{yan2023bad}, are defined as functions that can be exclusively invoked by privileged accounts that may affect other accounts without notification or authorization. Based on summaries of security incidents and feedback from smart contract developers, these papers identify various centralized risks. Our paper focuses on five such risks that can cause financial losses to contract users:

\textbf{\textit{Arbitrarily Transfer}.} \textit{Arbitrarily Transfer} enables privileged accounts, such as the owner of the smart contracts, to transfer tokens from other accounts without any approval. 
Taking Fig. \ref{fig:backdoor_transfer} as an example, the function \texttt{owner\_transfer} can be used to transfer tokens of account \texttt{from} to other accounts without its approval. 
This function is constrained by the \texttt{onlyOwner} modifier, restricting its invocation solely to the contract owner.

\begin{figure}[h]
\centering
\includegraphics[scale=0.4]{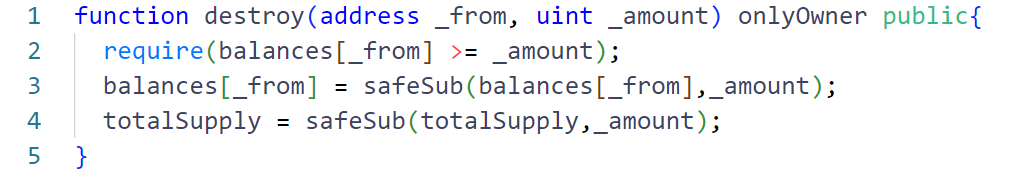}
\caption{An example of \textit{Destroy Account}.}
\label{fig:backdoor_destroy}
\end{figure}

\textbf{\textit{Destroy Account}.} Furthermore, aside from transferring tokens, a privileged account can exploit the \textit{Destroy Account} risks to destroy tokens of other accounts. 
As illustrated in Fig. \ref{fig:backdoor_destroy}, the contract owner can invoke the function \texttt{destroy}  to destroy tokens held by \texttt{\_from} account, resulting in the asset losses of the account.

\begin{figure}[h]
\centering
\includegraphics[scale=0.26]{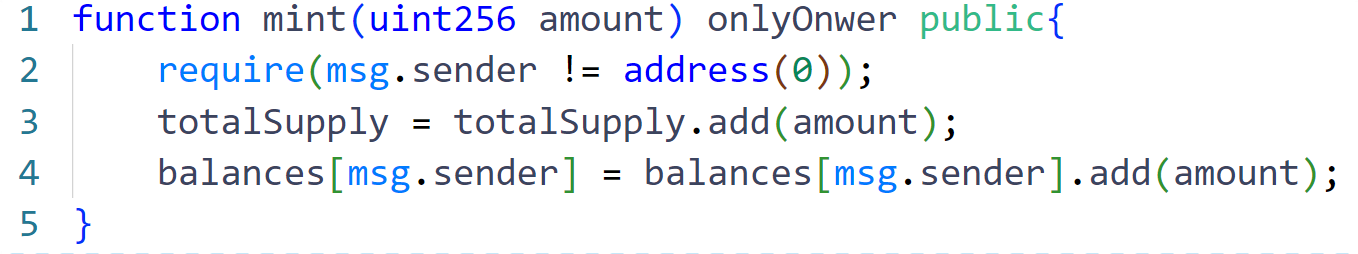}
\caption{An example of \textit{Arbitrarily Mint}.}
\label{fig:backdoor_mint}
\end{figure}

\textbf{\textit{Arbitrarily Mint}.}  \textit{Arbitrarily Mint}  enables certain privileged accounts to increase their token balances arbitrarily, potentially leading to  the inflating of total token supply and the decrease in the value of the associated tokens. 
An example of this kind of risks is shown in Fig. \ref{fig:backdoor_mint}, with the function \texttt{mint} exclusively accessible to the contract owner.

\begin{figure}[h]
\centering
\includegraphics[scale=0.4]{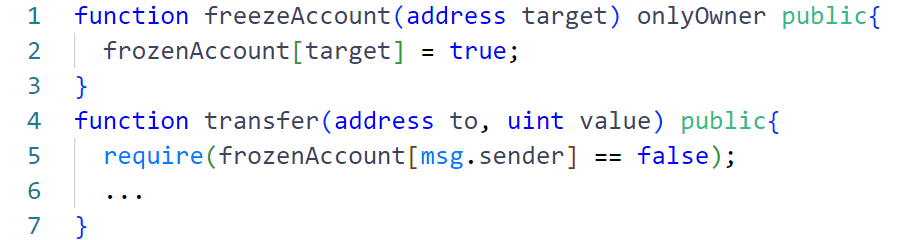}
\caption{An example of \textit{Freeze Account}.}
\label{fig:backdoor_freeze}
\end{figure}

\textbf{\textit{Freeze Account}.} \textit{Freeze Account}  can be leveraged to forbid transfers from a specific account. 
As depicted in Fig. \ref{fig:backdoor_freeze}, the contract owner can designate to freeze transfers from the \texttt{target} account.

\begin{figure}[h]
\centering
\includegraphics[scale=0.4]{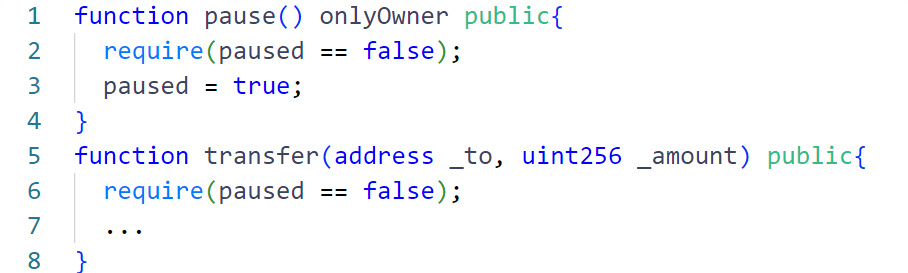}
\caption{An example of \textit{Disable Transferring}.}
\label{fig:backdoor_pause}
\end{figure}

\textbf{\textit{Disable Transferring}.} Similar to \textit{Freeze Account}, a \textit{Disable Transferring} risk can stop all users from transferring tokens. 
As illustrated in Fig. \ref{fig:backdoor_pause}, setting the value of \texttt{paused} to \texttt{true} blocks all users' transferring, with the \texttt{paused} variable exclusively modifiable by the contract owner.

Note that while \cite{ma2023pied} characterizes these risks as backdoor threats, we adopt the term `centralized risk' proposed in \cite{yan2023bad} throughout our paper.

% These aforementioned backdoors contravene the decentralized design ethos of smart contracts and blockchain technology by introducing a central node (i.e., privileged accounts), thereby engendering potential risks. For instance, although the mining function depicted in the figure offers a flexible token supply mechanism, it operates without the requisite permission of an ordinary account and lacks regulation, thereby susceptible to abuse, potentially resulting in loss of benefits for ordinary accounts.

% The aforementioned backdoors share a common feature: the introduction of unregulated privileged accounts, in other words, the introduction of central nodes. 
% This contradicts the fundamental concept in smart contract design, decentralization, thereby posing potential risks.

%% file: design.tex
%!TEX root = submission.tex

\section{\ourtool}

\subsection{Overview}
\label{subsec:overview}

The aforementioned risky codes exhibit two common features:
\begin{itemize}
\item They are exclusively accessible by privileged accounts, e.g., owners.
\item Their ultimate impact is the manipulation of financial variables, e.g., balances, leading to potential financial losses for users.
\end{itemize}

Based on these features, we propose the framework shown in Fig. \ref{fig:overview} to detect centralized risks:

According to the first feature, there are differences between the states reached by executing risky codes using privileged and ordinary accounts.
Given a contract, we perform symbolic execution of its functions using both privileged and ordinary accounts, comparing the outcome states.
Then, based on the differences obtained, we further search for the function related to the differences and perform the next round of execution.
This iterative process continues until no new differences emerge, culminating in a final difference set.
%This iterative process terminates until no new difference emerges, and outputs a final difference set.
%Additionally, in the process, we abstract the contract states and transform them into variable summaries, to ensure the termination of the iteration.

According to the second feature, the executions of risky codes affect the values of financial variables.
Thus, after obtaining the difference set, we perform financial variable recognition.
If differences related to financial variables are identified, indicating that the execution disparity between privileged and ordinary accounts impacts certain accounts' financial assets, we conclude the given contract is risky.

\begin{figure*}
\centering
\includegraphics[scale=0.35]{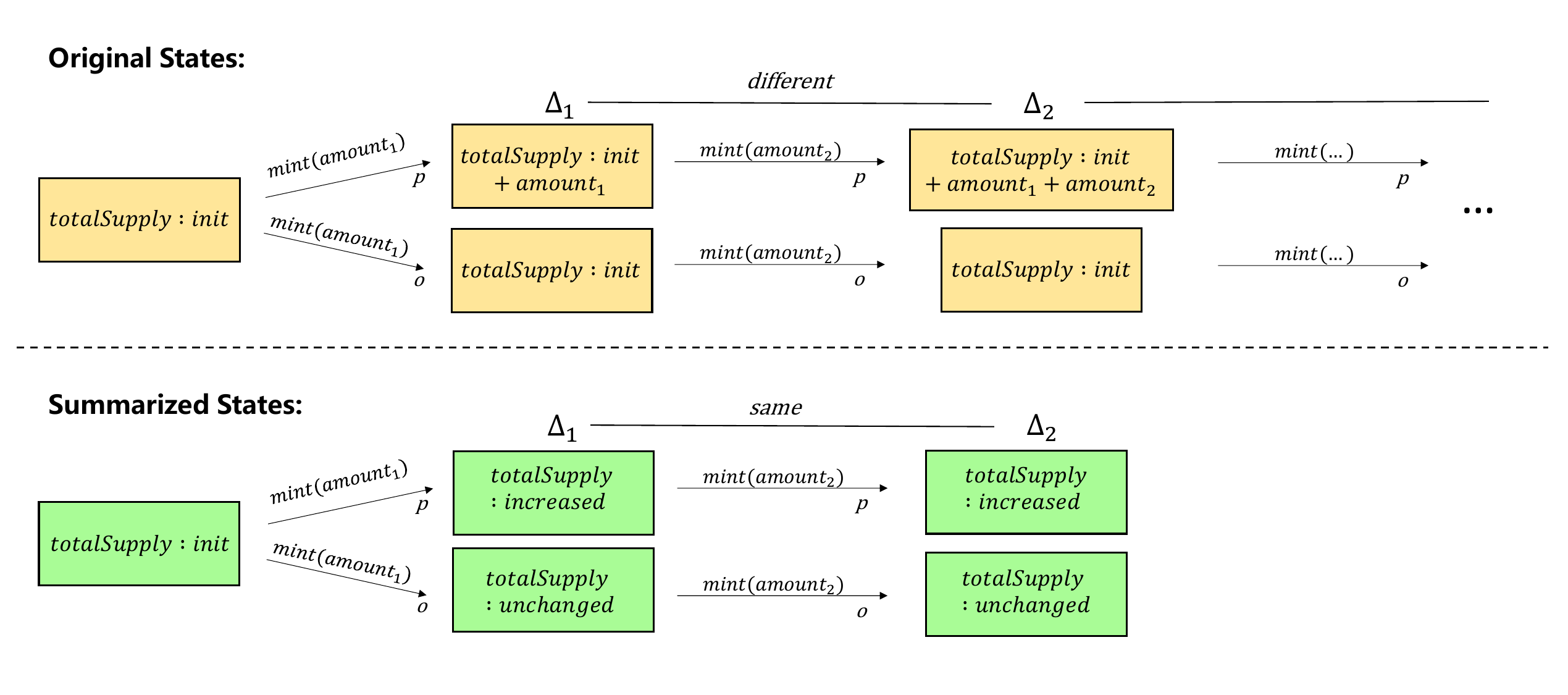}
\caption{An example of the non-termination problem of our iterative algorithm.}
\label{fig:problem}
\end{figure*}

We introduce the following notations for subsequent sections:
Given a contract $C$, we use $V_C$ to denote the set of state variables of $C$.
Denote $s=(\sigma,\pi)$ as a state of $C$, where $\sigma$ is a map from variable names to values and $\pi$ denotes the sequence of statements executed since the last state.
The transition from state $s$ to $s'$ is written as $s \stackrel{f(M)}{\longrightarrow}_{x} s'$.  
Here, $f(M)$ represents the executed function $f$ with parameters $M$ which leads to the transition, and $x \in \{p,o\}$ denote whether a privileged or ordinary account invokes $f$.

\subsection{Technical Challenges}

The technical challenges of our framework primarily lie in the iterative algorithm for obtaining differences: 

First, our framework faces the state space explosion problem.
As illustrated in Fig. \ref{fig:states}, each function execution using privileged accounts and ordinary accounts, denoted by arrows with subscripts \textit{p} and \textit{o} respectively, results in two states. 
To reach any two states in the same layer of the state space separately, at least one function must be executed using accounts with different permissions. 
Consequently, to fully capture the impact of different accounts on the states, we need to compare the states in each layer of the state space pairwise. 
For example, in Layer 4 of Fig. \ref{fig:states}, 28 comparisons (denoted by dotted lines) are required for 8 states. 
As the number of functions increases, the number of states grows exponentially, along with the number of comparisons. 
This exponential growth may ultimately prevent our method from completing state traversal and comparison within limited resources.

\begin{figure}
%\centering
\flushright
\includegraphics[scale=0.33]{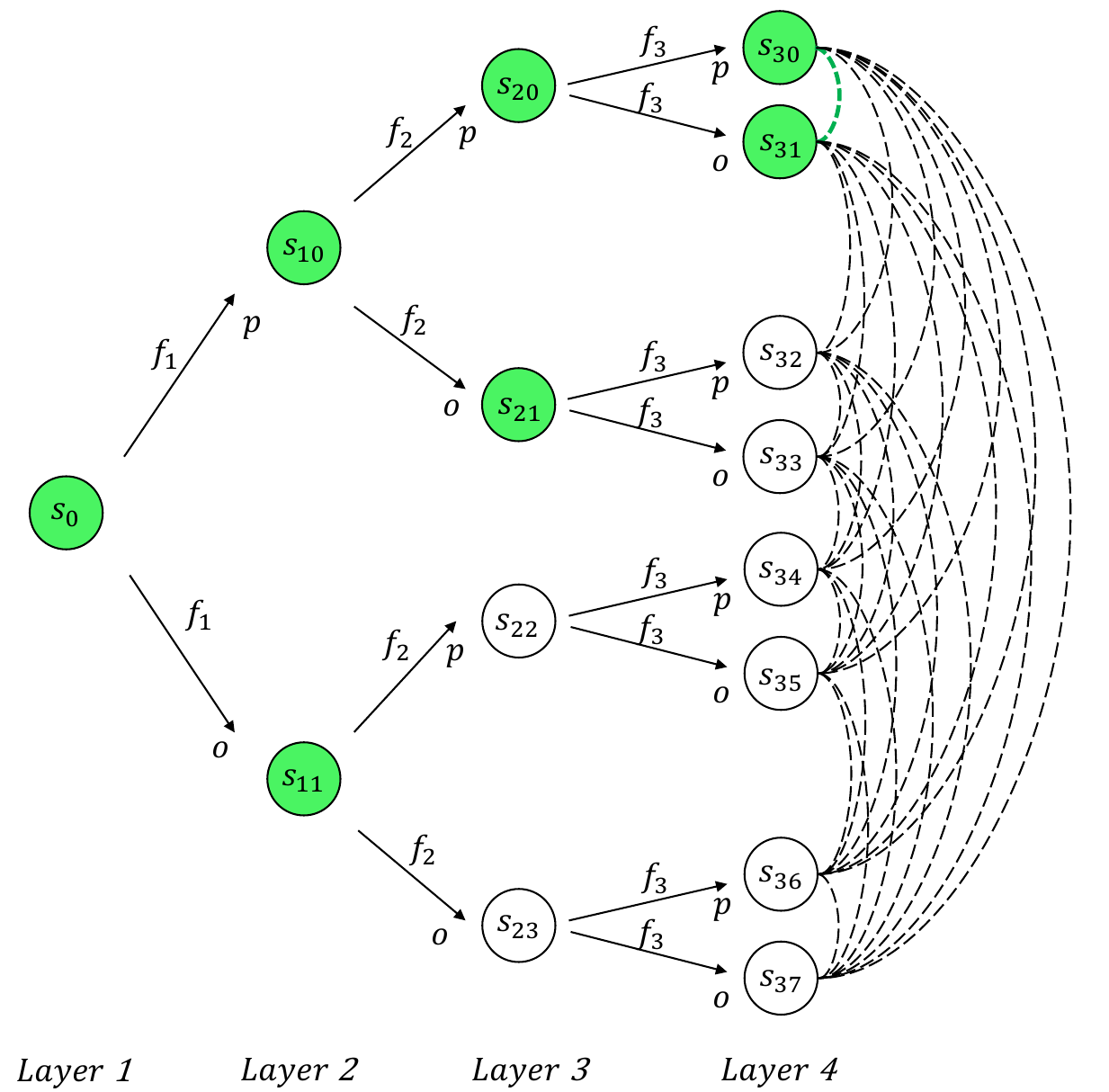} 
\caption{An example of the state space explosion problem of our iterative algorithm.}
\label{fig:states}
\end{figure}
% Starting from an initial state (denoted by the leftmost node), executing a function $f_1$ as a privileged and an ordinary account respectively results in two distinct states. 
% Subsequent execution of function $f_2$ from these two states doubles the number of resultant states. This exponential growth in the number of states with each function execution potentially causes our method to fail in completing state traversal and comparison.

Moreover, since a smart contract's functions can be called any number of times, leading to an infinite number of states, the iterative process of collecting differences may not terminate. 
For instance, consider the contract depicted in Fig.~\ref{fig:backdoor_mint}, where the \texttt{mint} function can be invoked by a privileged account for token generation, while an ordinary account lacks this capability. Assuming the initial value of the variable \texttt{totalSupply} is \textit{init}, executing the function under identical conditions but with different accounts may yield the path depicted in the red box in Fig.~\ref{fig:problem}.
%an arrow denotes a function execution, with subscripts \textit{p} and \textit{o} indicating privileged and ordinary account callers, respectively. \texttt{mint} above the arrows denotes the function name, while $amount_1$ and $amount_2$ symbolize the parameters. 
Here, $\Delta_1$ and $\Delta_2$ represent the differences in results after one and two function executions, respectively.
Following our iterative algorithm, we first obtain the difference $\Delta_1$ of the initial round (note that we omit sequence $\pi$ and the variables other than \texttt{totalSupply} for brevity). Since the \texttt{mint} function reads the values of \texttt{totalSupply}, we consider it a related function of $\Delta_1$ and initiate a second round of execution. 
The resulting $\Delta_2$ differs from $\Delta_1$ and is therefore considered a new difference, necessitating a third round of function execution. 
This process continues indefinitely because privileged accounts can keep modifying the value of \texttt{totalSupply}, unlike ordinary accounts, perpetuating the generation of new differences and preventing termination of the iterative process.

\begin{algorithm}
\caption{The Iterative Algorithm for Obtaining Differences in Contract States}
\label{alg:alg1}
\textbf{Input:}  source code $C$ of a contract \\
\textbf{Output:}  differences between the states reached by executing $C$ with privileged account $a_p$ and ordinary account $a_o$
\begin{algorithmic}[1]
\State $F \leftarrow$ the functions in $C$
\State $s_0 \leftarrow$ initial state after executing the constructor of $C$
\State $D \leftarrow \emptyset$
\State $S_{next} \leftarrow \emptyset$
\For{$f$ in $F$}
\State $s_p \leftarrow symbolic\_exec(f,a_p,s_0)$
\State $s_o \leftarrow symbolic\_exec(f,a_o,s_0)$
\State $delta \leftarrow \Delta(\Phi(s_p),\Phi(s_o))$
\If{$delta \neq \emptyset$ and $delta \notin D$}
\State $D.add(delta)$
\State $S_{next}.add(s_p)$
\EndIf 
\EndFor
\State $D_{new} \leftarrow D$
\State $s \leftarrow s_0$
\While{$D_{new} \neq \emptyset$}
\State $D_{new}' \leftarrow \emptyset$
\State $S_{next}' \leftarrow \emptyset$
\For{$delta$ in $D_{new}$}
\State $s \leftarrow S_{next}[D_{new}.index(delta)]$
\State $F' \leftarrow related\_funcs\_search(F,delta)$
\State $tmp \leftarrow \emptyset$
\State $S_{tmp} \leftarrow \emptyset$
\For{$f$ in $F'$}
\State $s_p \leftarrow symbolic\_exec(f,a_p,s)$
\State $s_o \leftarrow symbolic\_exec(f,a_o,s)$
\State $delta' \leftarrow \Delta(\Phi(s_p),\Phi(s_o))$
\If{$delta' \neq \emptyset$ and $delta' \notin tmp$}
\State $tmp.add(delta')$
\State $S_{tmp}.add(s_p)$
\EndIf
\EndFor
\If{$tmp \not\subseteq D$}
\State $D_{new}'.extend(tmp \setminus D)$ 
\State $D.extend(tmp \setminus D)$
\State $S_{next}'.update(S_{tmp},tmp,D)$ 
\EndIf
\EndFor
\State $D_{new} \leftarrow D_{new}'$
\State $S_{next} \leftarrow S_{next}'$
\EndWhile
\end{algorithmic}
\end{algorithm}

% \begin{algorithm}
% \caption{Iterative Algorithm for Obtaining Differences in Execution Results}
% \label{alg:alg1}
% \textbf{Input:}  source code $C$ of a contract \\
% \textbf{Output:}  the differences bewtween the results of executing $C$ by a privileged account $a_p$ and an ordinary account $a_o$
% \begin{algorithmic}[1]
% \State $F \leftarrow$ the functions in $C$
% \State $s_0 \leftarrow$ initial state after executing the constructor of $C$
% \State $D \leftarrow \emptyset$
% \For{$f$ in $F$}
% \State $s_p \leftarrow symbolic\_exec(f,a_p,s_0)$
% \State $s_o \leftarrow symbolic\_exec(f,a_o,s_0)$
% \State $delta \leftarrow \Delta(\Phi(s_p),\Phi(s_o))$
% \If{$delta \neq \emptyset$ and $delta \notin D$}
% \State $D.add(delta)$
% \EndIf 
% \EndFor
% \State $D_{new} \leftarrow D$
% \State $s \leftarrow s_0$
% \While{$D_{new} \neq \emptyset$}
% \State $D_{new}' \leftarrow \emptyset$
% \For{$delta$ in $D_{new}$}
% \State $s \leftarrow update(s,delta)$
% \State $F' \leftarrow related\_funcs\_search(F,delta)$
% \State $tmp \leftarrow \emptyset$
% \For{$f$ in $F'$}
% \State $s_p \leftarrow symbolic\_exec(f,a_p,s)$
% \State $s_o \leftarrow symbolic\_exec(f,a_o,s)$
% \State $delta' \leftarrow \Delta(\Phi(s_p),\Phi(s_o))$
% \If{$delta' \neq \emptyset$ and $delta' \notin tmp$}
% \State $tmp.add(delta')$
% \EndIf
% \EndFor
% \If{$tmp \not\subseteq D$}
% \State $D_{new}'.extend(tmp \setminus D)$ 
% \State $D.extend(tmp \setminus D)$
% \EndIf
% \EndFor
% \State $D_{new} \leftarrow D_{new}'$
% \EndWhile
% \end{algorithmic}
% \end{algorithm}

\subsection{Our Iterative Algorithm}
\label{subsec:algorithm}

To address the aforementioned challenges, we propose two methods used in the iterative algorithm: difference-oriented state traversal and variable summaries. 
In the following, we first introduce the overall flow of the algorithm, then explain how difference-oriented state traversal is performed within this flow, and finally present the role and categorization of variable summaries used in the algorithm.

\textbf{1) The Flow of the Iterative Algorithm}

Algorithm \ref{alg:alg1} illustrates the overall flow of the iterative algorithm, which takes the smart contract source code as input and produces the differences between the states reached by privileged accounts and ordinary accounts as output. 
The detailed flow is as follows:

\begin{itemize}

\item Initialization.
Retrieve all functions from the smart contract source code for analysis (Line 1). 
Obtain the initial state, $s_0$, by symbolically executing the contract's constructor using a symbolic execution engine implemented based on Slither\cite{feist2019slither} and Z3 \cite{Z3} (Line 2).
Initialize an empty ordered set $D$ to store differences (Line 3). 

\item Difference Set Initialization.
For each function (Line 5), conduct symbolic execution using both privileged and ordinary accounts (Lines 6-7).
Summarize and compare execution results to derive differences (Line 8).
$\Delta$ outputs the difference between two states, it will be introduced in the parts about variable summary along with $\Phi$.
Aggregate resultant differences from all functions into a set (Lines 9-10) for use in subsequent difference propagation.
Here, function $add$ is used to add an element into a set.

\item Difference Propagation.
Begin an iterative propagation process, continuing until no new differences are generated in the previous round (i.e., when $D_{new}$ is empty).
For each difference in $D_{new}$, first search for related functions (Line 21).  
The difference between two states consists of the difference between the individual variables in the two states, and its related functions are the functions that depend on the variables whose values differ between the two states.
We implement the search for related functions of differences using Slither.
Then for each related function (Line 24), conduct symbolic execution using privileged and ordinary accounts (Lines 25-26) and compare execution results to obtain differences (Line 27).
If there are newly discovered differences (Lines 33), add them to the $D_{new}$ set (Lines 34 and 39) for processing in subsequent rounds.
Here, $extend$ is used to combine two sets.
Finally, this iteration ends up with a converged difference set.

\end{itemize}

\textbf{2) Difference-oriented State Traversal} 

We categorize the differences between two states at any layer of the state space into two types:  those caused solely by differences in account permissions, and those stemming from differences in preceding states. 
Leveraging these distinct categories, we introduce a difference-oriented state traversal method to reduce the number of states requiring traversal by combining pruning and difference labeling while ensuring no loss of essential differential information.

First, for each layer in the state space, we selectively compare two states resulting from the execution of privileged and ordinary accounts, respectively. 
Specifically, we implement this step as follows:
In Difference Set Initialization of Algorithm \ref{alg:alg1}, we collect only the state $s_p$ (resulting from privileged account executions) and add them to the $S_{next}$ set (line 11).
In Difference Propagation of Algorithm \ref{alg:alg1}, for each difference $delta$ in $D_{new}$, we use the corresponding state $S_{next}[D_{new}.index(delta)]$ as the starting point for the symbolic execution of current round. 
Here, $D_{new}.index(delta)$ represent the index of $delta$ in set $D_{new}$.
For each function $f$ to be executed (Line 24), we also collect only the state $s_p$, adding it to both the temporary state set $S_{tmp}$ and $S_{next}$ as the initial state for the next round.
Differences in subsequent states of $s_{o}$ are preserved through labeled variables.

Second, motivated by the taint analysis technique \cite{newsome2005dynamic}, we add labels for variables exhibiting differences during comparison. 
This method preserves the second category of differences without increasing the number of traversed nodes.
To implement this step, we replace the states in Algorithm \ref{alg:alg1} with labeled state $l=(\sigma,\pi,\theta)$, where $\sigma$ is a map from variable names to values, $\pi$ denotes the sequence of statements executed since the last state, and $\theta$ is a map from variable names to boolean values indicating whether the variable is labeled.
Similar to the transition of states, the transition from labeled state $l$ to $l'$ is written as $l \stackrel{f(M)}{\longrightarrow}_{x} l'$.
Due to the page limit, we present the modified algorithm in \cite{wang2024janus}.
Formal proofs for this mechanism are provided in Section \ref{subsec:proof}.

Taking Fig. \ref{fig:states} as an example, we traverse and compare only the states represented by green nodes, thereby preserving the first category of differences.
Then, to preserve the second category of differences, when comparing states $s_{10}$ and $s_{11}$, we add labels to variables with different values. 
Subsequently, we conduct a data flow analysis of function $f_2$ and add labels for variables dependent on the labeled ones. 
If a labeled variable undergoes reassignment in $f_2$, and the assignment statement is independent of any labeled variable, we remove the label from that variable.
In the comparison of states $s_{20}$ and $s_{21}$ in Layer 3, labeled variables are not directly compared but are instead incorporated into the set of difference variables. 
These labeled variables, which display value differences between states $s_{22}$ or $s_{23}$ and $s_{20}$, are included in the subsequent analysis of financial variables.
Note that although $s_{22}$ and $s_{23}$ remain untraversed, labeling difference variables during the comparison of $s_{10}$ and $s_{11}$ ensures that when these nodes exhibit variable differences with other nodes in the same layer, the comparison of $s_{20}$ and $s_{21}$ will produce identical differences.

\begin{table*}[]
 \centering 
 \caption{The Design of Variable Summaries for Different Variable Types}
 \label{tab:sum}
\begin{tabular}{c|c|c|c|c}
\hline
\textbf{Variable Type}                             & \textbf{Key}         & \textbf{Category}        & \textbf{Summarized Value}  & \textbf{Description} \\ \hline
\multirow{3}{*}{Numeric}         & \textit{is\_increased}        & \uppercase\expandafter{\romannumeral 1}  & \{True,False\}    &  Whether the variable's value is increased compared to the last state.           \\ \cline{2-5} 
                             & \textit{is\_descreased}        & \uppercase\expandafter{\romannumeral 1}   & \{True,False\}   &  Whether the variable's value is decreased compared to the last state.           \\ \cline{2-5} 
                             & \textit{related\_const\_var}   & \uppercase\expandafter{\romannumeral 2}  & State variables and constants &  The state variables and constants related to the variable's value.           \\ \hline
\multirow{3}{*}{Address}        & \textit{is\_constant}     & \uppercase\expandafter{\romannumeral 1} & \{True,False\}    &  Whether the variable's value is assigned as a constant address.              \\ \cline{2-5} 
                             & \textit{is\_changed}      & \uppercase\expandafter{\romannumeral 1}  & \{True,False\}   &  Whether the variable's value is changed compared to  the last state.   \\ \cline{2-5}         
                             & \textit{related\_const\_var}  & \uppercase\expandafter{\romannumeral 2} & State variables and constants   &  The state variables and constants  related to the variable's value.           \\ \hline
Mapping/ether                             & \textit{key variable set}          & /   & Summaries of value variables  &  The summaries of value variables corresponding to their types.      \\ \hline
\multirow{3}{*}{exec\_state} & \textit{success}                 & \uppercase\expandafter{\romannumeral 2} & \{True,False\}  & Whether the function is executed successfully.          \\ \cline{2-5} 
                             & \textit{revert}                 & \uppercase\expandafter{\romannumeral 2}  & \{True,False\} & Whether the function is reverted during execution.          \\ \cline{2-5} 
                             & \textit{selfdestruct}            & \uppercase\expandafter{\romannumeral 2}  & \{True,False\}  & Whether the function leads to selfdestruct of the contract.         \\ \hline

\end{tabular}
\begin{tabular}{c}
Category \uppercase\expandafter{\romannumeral 1}: generated based on variable values\ \ \ \ Category \uppercase\expandafter{\romannumeral 2}: generated based on the executed statements \\
\end{tabular}
\end{table*}

\textbf{3) Variable Summary}

In both Line 7 and 23 of Algorithm 1, when comparing the differences in states $s_p$ and $s_o$, we first generate a summary for them. 
This abstraction of variable information reduces the number of state differences. 
Using the path depicted in the red box in Fig. \ref{fig:problem}  as an example, if we summarize the states reached by executing \texttt{mint} as the growth trend of variable \texttt{totalSupply}, we get the path shown in the green box in Fig. \ref{fig:problem}.
Regardless of the function's execution frequency, the result differences can be consistently described: the value of \texttt{totalSupply} increases for the privileged account while remaining unchanged for the ordinary account. Consequently, as execution times increase, no new differences emerge, allowing the iterative process to terminate.

% To address the aforementioned problem, we propose the variable summary, which abstracts pertinent information concerning a variable.
% For example, in the above scenario, summarizing the variable \texttt{totalSupply} as exhibiting a growth trend suffices. 
% Regardless of the function's execution times, the differences in results can be described as follows: \texttt{totalSupply} increases in the outcome for the privileged account, while it remains constant for the ordinary account. 
% Consequently, as the execution times increase, no new difference emerges, and the iterative algorithm can be terminated, as shown in Fig. \ref{fig:problem} (b).

Specifically, we summarize the information of a variable into a set of key-value pairs, where different types of variables correspond to different numbers of keys. 
As shown in TABLE \ref{tab:sum}, we design corresponding variable summaries for the following variable types:
\begin{itemize}
\item Numeric Type.
The summary of Numeric Type variables comprises three keys used to denote variable growth and data interdependency with other state variables.
\item Address Type.
The summary of Address Type variables consists of two keys describing the mutability of the address variable and whether it is set to a specific value.
\item Mapping Type.
A Mapping Type variable is a map from a set of key variables to a set of value variables. The summary of a variable of this type has its keys as the set of key variables and its values as the summary of the individual value variables. For instance, consider the \texttt{mapping(address=>uint)} variable called \texttt{balances} in Fig. \ref{fig:backdoor_transfer}. Assume that the key variables of \texttt{balances} are \texttt{\_from} and \texttt{\_to}, and the value variables are denoted by \texttt{balances[\_from]} and \texttt{balances[\_to]}. Given a transition from labeled state $l$ to $l'$ by executing \texttt{owner\_transfer}, we have
{\small
\begin{align}
\phi(l,l',\texttt{balances}) =  \{\texttt{\_to}:\phi(l,l',\texttt{balances[\_to]}  )\notag \\ ,\texttt{\_from}:\phi(l,l',\texttt{balances[\_from]})\} \notag
\end{align}}
Here, \texttt{balances[\_from]} and \texttt{balances[\_to]} are treated as Numeric Type variables, and their summaries are generated according to TABLE \ref{tab:sum}.
Function $\phi(l,l',v)$ outputs the summary of variable $v$ under the transition from $l$ to $l'$.
\item Other Types.
Two additional types of variables, namely \texttt{exec\_state} and \texttt{ether}, are automatically generated during the symbolic execution of contracts. The \texttt{exec\_state} is a string-type variable that can take on three possible values: $success$, $revert$, or $selfdestruct$. 
These values respectively signify whether the function execution succeeds, reverts, or causes the contract to self-destruct.
Each possible value of \texttt{exec\_state} corresponds to a key in its summary. 
When \texttt{exec\_state} takes on a specific value, the corresponding key in the summary is set to true, while the others remain false. During the executions, a variable \texttt{ether} of type \texttt{mapping(address=>uint)} is employed to record the ether balances of accounts. This variable is handled in the same way as the mapping variable described above.
\end{itemize}
Note that since the values of Boolean Type variables can only be \texttt{true} or \texttt{false}, we do not design summaries to abstract the information of them.

Furthermore, we categorize the key-value pairs in the summaries into two types: those generated based on variable values, and those generated based on the executed statements. 
For example, \textit{is\_increased} and \textit{is\_decreased} of Numeric Type variables fall into the former category, as they are determined by changes in the variable's value.
Conversely, \textit{related\_const\_var} belongs to the latter category, as it needs to be generated according to data dependencies obtained from executed statements. 
The specific categorization of each key-value pair is shown in TABLE \ref{tab:sum}.
Since both the summary types and the number of state variables in a contract are fixed, the maximum number of elements in the final difference set is bounded, and the iterative algorithm is theoretically guaranteed to terminate. Consequently, we do not set a maximum iteration count for the algorithm.

Based on the above definitions, we extend the function $\phi$ as follows:
$\phi(l,l',v) = \phi((\sigma,\pi,\theta),(\sigma',\pi',\theta'),v) = \phi_{\sigma}(\sigma,\sigma',\theta',v) \cup \phi_{\pi}(\pi',\theta',v)$.
Here, $\phi_{\sigma}$ generates the summarized values based on variable values, while $\phi_{\pi}$ generates the summarized values based on the executed statements.
In particular, $\theta'$ is used in both functions since the labels of variables are related to both variable values and executed statements.
Then we define $\Phi(l,l') = \{\ \phi(l,l',v)\ |\ \forall v \in V_C  \}$.
We compute the difference $\Delta(\Phi(l,l'),\Phi(l,l''))$ by comparing $\phi(l,l',v)$ and $\phi(l,l'',v)$ for all $v \in V_C$.
If for all $v \in V_C$,  $\phi(l,l',v)$ = $\phi(l,l'',v)$ and $l'.\theta(v) = false$, we consider $\Phi(l,l')$ = $\Phi(l,l'')$ and the difference is empty.
Otherwise, $\Delta(\Phi(l,l'),\Phi(l,l'')) = \{\phi(l,l',v)\ |\ v \in V_C. (\phi(l,l',v) \neq \phi(l,l'',v) \lor l'.\theta(v) = true) \}$.

%% file: design2.tex
\subsection{Financial Variable Recognition}
\label{subsec:financial_var}

After collecting the differences, we need to identify whether they contain financial variables. 
To achieve this, we propose heterogeneous variable property graphs to represent the semantics of smart contracts, based on the approach of representing programs through graphs in \cite{liu2021combining} and \cite{zhou2019devign}.
The nodes of heterogeneous variable property graphs are divided into three categories: State Variable Nodes, Local Variable Nodes, and Function Statement Nodes. We distinguish between state variables and local variables because local variables are temporarily saved during function execution and do not appear in the final execution result. 
In other words, our goal is to find finance-related state variables.
The edges of heterogeneous variable property graphs are divided into six categories:
\begin{itemize}
\item Control Flow Edges (CFE): Basic edges between Function Statement Nodes, representing control flow relationships.
\item Data Flow Edges (DFE): Connecting Function Statement Nodes to variable nodes, representing variables read and written by statements.
\item Reference Flow Edges (RFE): Representing the data flow relationship between function parameters and arguments.
\item Control Dependency Edges (CDE): Indicating which variables constrain a Function Statement Node.
\item Data Dependency Edges (DDE): Representing dependencies between different variable nodes.
\item Function Call Edges (FCE): Expressing function call relationships between different Function Statement Nodes.
\end{itemize}
Then, we use graph neural networks \cite{wu2022graph}, an effective method for graph representation learning, to classify the nodes in the constructed graphs to identify financial variables from smart contracts. 
Based on the experimental results shown in Section \ref{subsec:comp_var}, we choose the commonly used Graph Convolutional Network (GCN) \cite{zhang2019graph} and Graph Attention Network (GAT)~\cite{velivckovic2017graph}. 
Our network consists of 5 layers of GCN and 1 layer of GAT, which convert each node in the variable property graphs into a feature vector. 
We then classify the nodes through a Multilayer Perceptron~\cite{popescu2009multilayer} and a Sigmoid function.
The impact of hyperparameters, e.g., the number of layers, is evaluated and demonstrated in \cite{wang2024janus}.
Finally, we check whether the differences contain the financial variables identified through the above process.

%% file: discuss.tex
%!TEX root = submission.tex

\subsection{Proof}
\label{subsec:proof}

First, we propose the following definitions:

\begin{definition}
\label{def1}
(Differential states) Given a smart contract $C$, if there exist states $s$, $s_p$ and $s_o$ that satisfy 

a) $s \stackrel{f(M)}{\longrightarrow}_{p} s_p$

b) $s \stackrel{f(M)}{\longrightarrow}_{o} s_o$ 

c) $\exists v \in V_C. s_p.\sigma(v) \neq s_o.\sigma(v)$

then we say that $s_p$ and $s_o$ are a pair of differential states of $C$, and $s$ is their source.
The variables that satisfy condition c) are called branch variables of  $s_p$ and $s_o$.
The set of all branch variables for states $s_p$ and $s_o$ is denoted by $BV(s_p,s_o)$.
\end{definition}

\begin{definition}
\label{def2}
(Labeled differential states) Given a smart contract $C$, if there exist labeled states $l$, $l_p$ and $l_o$ that satisfy 

a) $l \stackrel{f(M)}{\longrightarrow}_{p} l_p$

b) $l \stackrel{f(M)}{\longrightarrow}_{o} l_o$

c) $\exists v \in V_C. (l_p.\sigma(v) \neq l_o.\sigma(v) \lor l_p.\theta(v) = true)$

then we say that $l_p$ and $l_o$ are a pair of labeled differential states of $C$, and $l$ is their source.
The variables that satisfy condition c) are called branch variables of  $l_p$ and $l_o$.
The set of all branch variables for labeled states $l_p$ and $l_o$ is denoted by $BV(l_p,l_o)$.
\end{definition}

% \begin{definition}
% \label{def3}
% For a pair of differential states $s_p$ and $s_o$, we define a branch variable $v$ as one that satisfies the condition $s_p.\sigma(v) \neq s_o.\sigma(v)$. The set of all branch variables for states $s_p$ and $s_o$ is denoted by $BV(s_p,s_o)$.
% \end{definition}

We propose and prove the following theorems (informal) based on the above definitions: 
\begin{enumerate}
\item If a pair of differential states exists in the state space, then a pair of labeled differential states can be found using the difference-oriented state traversal method.
\item  If a pair of labeled differential states is found using the difference-oriented state traversal method, a pair of differential states must exist in the state space.
\item For a pair of labeled differential states, the summaries of all the branch variables are also different.
\end{enumerate}
If the first two theorems hold, it demonstrates that our difference-oriented state traversal method guarantees no omissions or false positives in the detection outcomes. Furthermore, if the third theorem holds, it confirms that the use of variable summaries does not influence the detection results, i.e., it does not alter the determination of whether centralized risks exist.
Due to the page limit, please refer to the \cite{wang2024janus} for the formal theorems and their proofs.

%% file: experiment.tex
\section{Experiment}
\label{sec:exp}

In this section, we conduct experiments to address the following questions regarding our tool's performance:
\begin{enumerate}[Q1.]
\item How does the effectiveness of \ourtool in detecting financial centralized risks in Solidity smart contracts compare to existing tools?
\item Can \ourtool successfully identify financial centralized risks in real-world smart contracts?
\item How accurately does the financial variable recognition module identify financial variables within contracts?
\end{enumerate}

\subsection{Experimental Setup}
\label{subsec:setup}

Our experiments are conducted on a server with a 2.40GHz CPU, 128GB of RAM, and an NVIDIA GeForce RTX 3090 graphics card, running Ubuntu 18.04.6 LTS.

Our experimental dataset comprises a \textit{comparison dataset} and a \textit{real-world dataset}: 

The \textit{comparison dataset} draws from two sources: 1) the open source dataset provided by Pied-Piper \cite{tool}, containing smart contracts categorized into five types of centralized risks: \textit{Arbitrarily Transfer}, \textit{Destroy Account}, \textit{Arbitrarily Mint}, \textit{Freeze Account}, and \textit{Disable Transferring}.  2) smart contracts with financial centralized risks exhibiting security events, totaling 9 according to \cite{event1} \cite{event2} \cite{event3} \cite{event4}.
% Note that we exclude contracts from the Pied-Piper dataset that cannot be compiled. 
Note that we compile each contract using the compiler version specified in its code and exclude non-compilable contracts. 
% This approach is necessary because Pied-Piper, the tool we compare against, requires bytecode as input.
From these sources, we obtain 270 contracts containing risks. We manually repair the risky code to produce 270 repaired versions of the contracts, i.e., risk-free contracts. Ultimately, the totaling 540 contracts, collectively form our \textit{comparison dataset}.

The \textit{real-world dataset}, sourced from Ethereum and Binance Smart Chain using a crawler \cite{crawler}, comprises 33,151 smart contracts deployed in real-world environments from various addresses.
Note that there is one duplicate contract between comparison dataset and real-world dataset. Due to the small number, no specific action was taken to address this contract.

\subsection{Experiment on the comparison dataset}
\label{subsec:comp}

\begin{table}[]
 \centering 
 \caption{comparative results of tools on comparison dataset}
 \label{tab:compare}
\begin{tabular}{c|c|cccc}
\hline
\multirow{2}{*}{\textbf{Risk Type}}        & \multirow{2}{*}{\textbf{Tool}} & \multicolumn{4}{c}{\textbf{Result}}                                                                                  \\ \cline{3-6} 
                                      &                       & \multicolumn{1}{c|}{\textbf{Total}}                & \multicolumn{1}{c|}{\textbf{FP}} & \multicolumn{1}{c|}{\textbf{FN}} & \textbf{Avg Time(s)} \\ \hline
\multirow{3}{*}{\textit{Arbitrarily Transfer}}      
                                      & Pied-Piper            & \multicolumn{1}{c|}{\multirow{3}{*}{80}}                     & \multicolumn{1}{c|}{0}  & \multicolumn{1}{c|}{3}  & 1.98 \\ \cline{2-2} \cline{4-6}        
                                      & Tokeer                & \multicolumn{1}{c|}{}                     & \multicolumn{1}{c|}{0}  & \multicolumn{1}{c|}{0}  & 0.62 \\ \cline{2-2} \cline{4-6} 
                                      & \ourtool                  & \multicolumn{1}{c|}{}   & \multicolumn{1}{c|}{0}  & \multicolumn{1}{c|}{0}  & 2.11       \\ \hline
\multirow{3}{*}{\textit{Arbitrarily Mint}}                 
                                      & Pied-Piper            & \multicolumn{1}{c|}{\multirow{3}{*}{70}}                     & \multicolumn{1}{c|}{1}  & \multicolumn{1}{c|}{4}  & 5.21        \\ \cline{2-2} \cline{4-6} 
                                      & Tokeer                & \multicolumn{1}{c|}{}                     & \multicolumn{1}{c|}{27} & \multicolumn{1}{c|}{5}  & 1.88     \\ \cline{2-2} \cline{4-6}    
                                      & \ourtool          & \multicolumn{1}{c|}{}         & \multicolumn{1}{c|}{0}  & \multicolumn{1}{c|}{0}  & 17.63       
                                      \\ \hline
\multirow{3}{*}{\textit{Destroy Account}}              
                                      & Pied-Piper            & \multicolumn{1}{c|}{\multirow{3}{*}{116}}                  & \multicolumn{1}{c|}{4}  & \multicolumn{1}{c|}{3}  & 4.16        \\ \cline{2-2} \cline{4-6} 
                                      & Tokeer                & \multicolumn{1}{c|}{}                     & \multicolumn{1}{c|}{14} & \multicolumn{1}{c|}{7} & 0.96    \\ \cline{2-2} \cline{4-6}     
                                      & \ourtool       & \multicolumn{1}{c|}{}            & \multicolumn{1}{c|}{1}  & \multicolumn{1}{c|}{0}  & 7.03         \\ \hline
\multirow{3}{*}{\textit{Disable Transferring}}               
                                      & Pied-Piper            & \multicolumn{1}{c|}{\multirow{3}{*}{124}}                     & \multicolumn{1}{c|}{7}  & \multicolumn{1}{c|}{3}  & 5.34        \\ \cline{2-2} \cline{4-6} 
                                      & Tokeer                & \multicolumn{1}{c|}{}                     & \multicolumn{1}{c|}{/}  & \multicolumn{1}{c|}{/}  & /          \\ \cline{2-2} \cline{4-6}  
                                      & \ourtool                  & \multicolumn{1}{c|}{} & \multicolumn{1}{c|}{0}  & \multicolumn{1}{c|}{0}  & 8.85       \\ \hline
\multirow{3}{*}{\textit{Freeze Account}}                    
                                      & Pied-Piper            & \multicolumn{1}{c|}{\multirow{3}{*}{150}}                     & \multicolumn{1}{c|}{56} & \multicolumn{1}{c|}{4}  & 6.65        \\ \cline{2-2} \cline{4-6} 
                                      & Tokeer                & \multicolumn{1}{c|}{}                     & \multicolumn{1}{c|}{6}  & \multicolumn{1}{c|}{14} & 2.86     \\ \cline{2-2} \cline{4-6}   
                                      & \ourtool                  & \multicolumn{1}{c|}{} & \multicolumn{1}{c|}{0}  & \multicolumn{1}{c|}{0}  & 30.01   \\ \hline
\end{tabular}
\end{table}

\begin{figure}[h]
\centering
\includegraphics[scale=0.38]{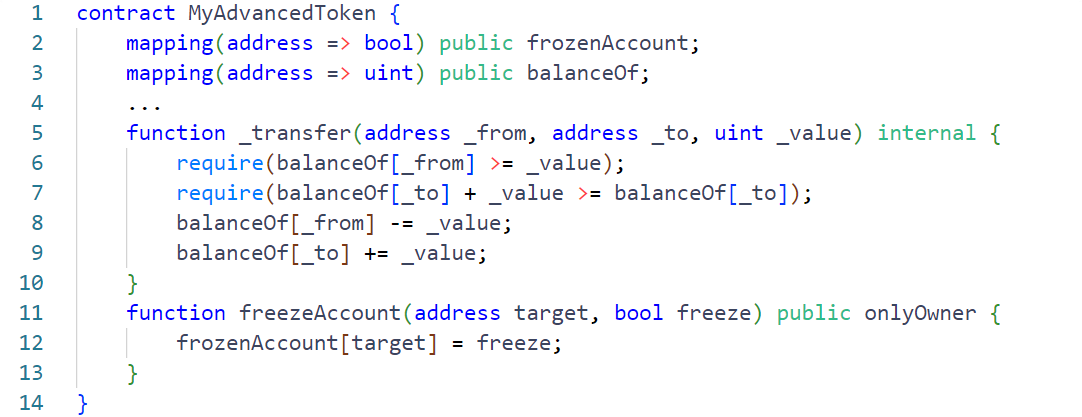}
\caption{An example of FP results detected by Pied-piper.}
\label{fig:Pied_FP}
\end{figure}

\begin{figure}[h]
\centering
\includegraphics[scale=0.38]{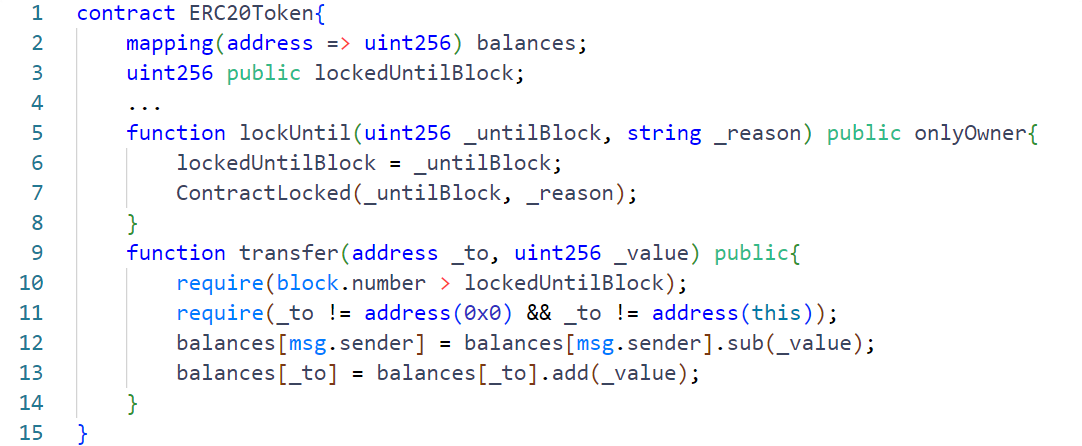}
\caption{An example of FN results detected by Pied-piper.}
\label{fig:Pied_FN}
\end{figure}

\begin{figure}[h]
\centering
\includegraphics[scale=0.36]{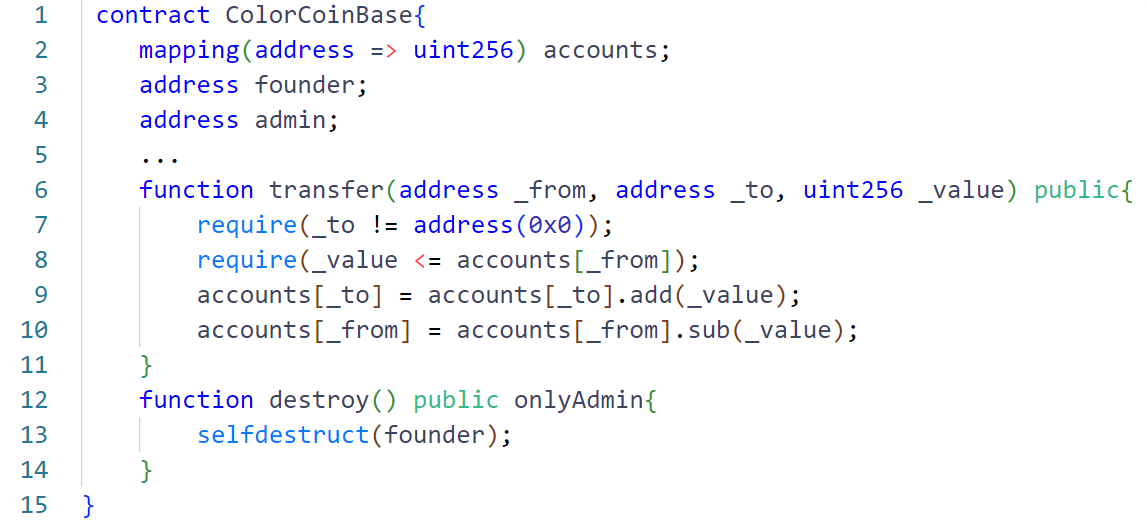}
\caption{An example of FP results detected by Tokeer.}
\label{fig:Tokeer_FN}
\end{figure}

\textbf{Results of \ourtool.} The results presented in TABLE \ref{tab:compare} indicate that \ourtool effectively identifies risks in the \textit{comparison dataset}, yielding 1 FP and 0 FN results across all five types of risks.
We manually check the FP result, finding a contract with a \texttt{destroy} function allowing the owner to arbitrarily destroy other accounts' token balances. 
Our tool flags this as a \textit{Destroy Account} risk. 
However, the contract lacks a token transfer function, rendering the balances economically worthless. 
Thus, we conclude this contract posing no financial harm. 
To avoid this FP result, we could use other existing tools to check token liquidity \cite{tsankov2018securify}. 
However, as defining a token balance variable without a transfer function is uncommon, we haven't implemented this token liquidity check in JANUS.

\textbf{Results of Pied-Piper.} In comparison, Pied-Piper outputs FP and FN results for all five types of risks.
The false results generated by Pied-Piper stem from inaccurate pre-defined patterns.
For instance, Pied-Piper incorrectly reports 56 FP results when identifying \textit{Freeze Account}.
Taking the contract depicted in Fig. \ref{fig:Pied_FP} as an example, it defines a function named \texttt{freezeAccount}, leading Pied-Piper to identify a \textit{Freeze Account} risk based on predefined patterns.
Nevertheless, the \texttt{freezeAccount} function does not impact the execution of the \texttt{transfer} function, i.e., does not impact users' financial assets.
Another illustration is the 3 FN results output by Pied-Piper, each containing a \textit{Disable transferring} risk, with one example depicted in Fig. \ref{fig:Pied_FN}.
In this contract, the owner could use the \texttt{lockUntil} function to lock the whole contract when the block number exceeds a certain threshold, instead of directly setting a Boolean variable to lock the contract, and thus bypass Pied-Piper's pre-defined pattern.

\textbf{Results of Tokeer.} Tokeer's goal is to detect rug pull contracts, so its predefined patterns are not entirely suitable for risk detection, resulting in detection results.
For example, in Fig. \ref{fig:Tokeer_FN}, a risk existing in the \texttt{destroy} function, which is callable only by the contract owner, can destroy the whole contract and cause losses of other accounts. 
However, Tokeer only detects whether the transfer flow of the contract is normal, i.e., if the token balance can be arbitrarily modified or if transfers can be prohibited.
It does not detect whether the contract can be destroyed, and deems the contract to be safe.

\begin{figure*}[h]
\centering
\begin{minipage}{\columnwidth}
\centering
\includegraphics[scale=0.37]{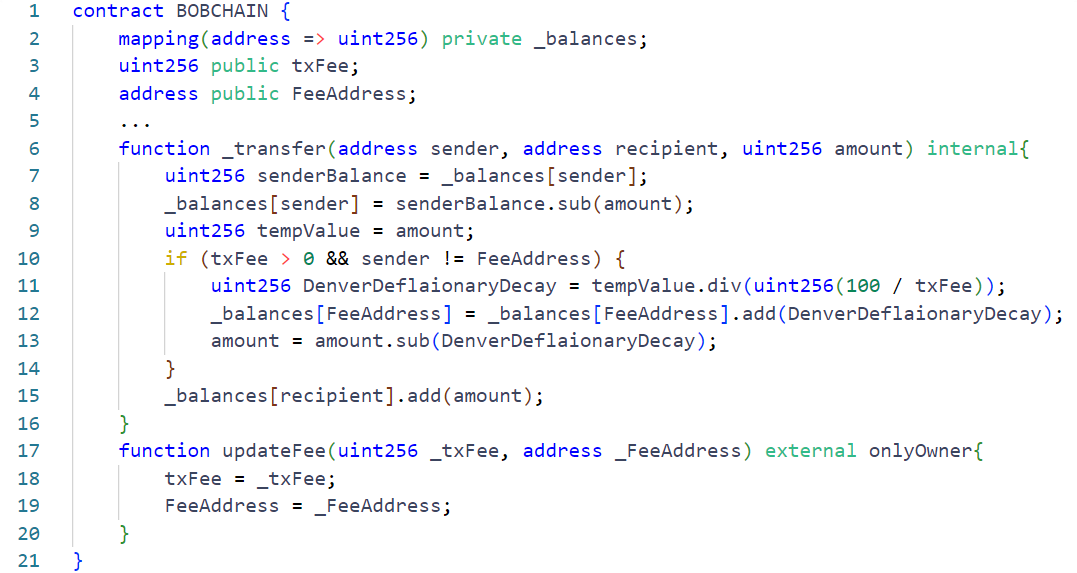}
\caption{An example of \textit{Parameter Manipulation}.}
\label{fig:param}
\end{minipage}
\begin{minipage}{.95\columnwidth}
\centering
\includegraphics[scale=0.39]{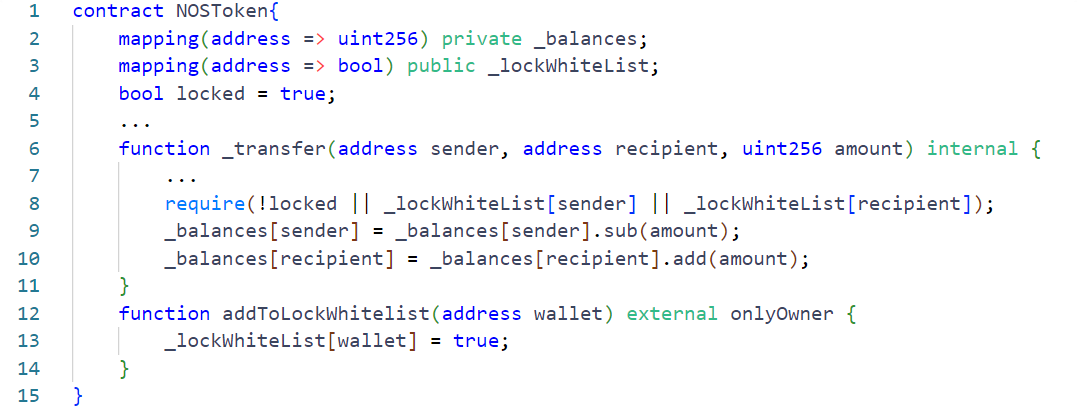}
\caption{An example of \textit{Whitelist}.}
\label{fig:whitelist}
\end{minipage}
\end{figure*}

We also compare the efficiency of tools in detecting centralized risks in smart contracts, and the results are shown in TABLE \ref{tab:compare}.
Since our tool requires symbolic execution of the contract, it takes more time compared to Pied-Piper and Tokeer, which utilize Datalog analysis.
The difference in average detection time between tools is related to the type of risk in the contract being detected.
Specifically, for contracts with \textit{Arbitrarily Transfer}, \textit{Destroy Account}, and \textit{Disable Transferring}, \ourtool exhibits an average detection time no more than 3 seconds longer than other tools. 
For contracts containing \textit{Arbitrarily Mint} and \textit{Freeze Account} risks, \ourtool demonstrates an average detection time surpassing that of other tools by more than 10 seconds but remaining under 30 seconds. 
Considering the advantage of our tool in terms of accuracy, we believe that the aforementioned time overheads are acceptable.

\subsection{Experiment on the real-world dataset}
\label{subsec:real}

To address question Q2, we evaluate  our tool on the \textit{real-world dataset} and conduct a manual check on the results.

Our tool detects 8391 smart contracts with financial centralized risks, with each contract taking an average of 112.44 seconds.
Due to the large number of contracts, we select a subset for inspection. Specifically, we randomly sample 500 contracts flagged as risky by our tool and another 500 contracts identified as non-risky. We invite three researchers in the field of smart contract security to independently review the detection results for these 1000 contracts, and their findings are aggregated. The results show that our tool achieved 455 TP and 457 TN among these 1000 contracts, with a Precision of 91.0\% and a Recall of 91.37\%. This demonstrates, to some extent, that our tool remains effective in detecting centralized risks in real-world contracts.

We analyze the tool's FP and FN results to investigate their causes. Among these, FPs are primarily caused by errors in the financial variable identification module. This module misidentifies some non-financial variables within the final difference set as financial variables, leading to FPs. Conversely, FNs are mainly attributable to our criteria for identifying privileged accounts. Based on common coding practices, we search for privileged accounts among state variables of the \texttt{address} type, consequently overlooking privileged accounts in some contracts. For example, certain contracts use variables of type \texttt{mapping(address => Role)} (where \texttt{Role} is a custom struct) to represent various accounts for refined role functionality, and our tool currently cannot recognize this category of privileged accounts.

Aside from the five known types of centralized risks, two new types of risks, \textit{Parameter Manipulation} and \textit{Whitelist}, are discovered during the manual review of the detection results output by our tool. 

A simplified contract with a \textit{Parameter Manipulation} risk is shown in Fig. \ref{fig:param}, and the address of the original contract is \cite{real1}.
In this contract, the sender of each transfer needs to pay a fee, which is specified by the variable \texttt{txFee}.
However, the contract owner can freely adjust this fee via the \texttt{updateFee} function, initially attracting users with a low transfer fee and later increasing it during transactions to obtain more tokens and revenue.
%etherscan.io/address/0xf5772555a24ae1f7528a1e09e03919578665dc72
% \begin{figure}[h]
% \centering
% \includegraphics[scale=0.41]{Whitelist.png}
% \caption{An example for \textit{Arbitrarily Transfer} backdoors.(TODO)}
% \label{fig:whitelist}
% \end{figure}
Additionally, a simplified example with a \textit{Whitelist} risk is illustrated in Fig. \ref{fig:whitelist}.
The address of the original contract is \cite{real2}. 
Upon deployment, the contract initializes the variable \texttt{locked} to \texttt{true}, thereby prohibiting all account-to-account transfers. 
However, as evidenced by the code in Line 12 and 13, the contract owner can utilize the \texttt{addToLockWhitelist} function to set a whitelist, granting specific accounts the  privileges to transfer tokens. 
This type of risks allows the contract owner to extend privileges to additional accounts, potentially causing more harm when combined with other risks.

%bscscan.com/address/0x7c397eB2003ae0C7392277c5679E3d3e6A733087

% https://gopluslabs.io/

We find 10 contracts with \textit{Parameter Manipulation}  and 10 contracts with \textit{Whitelist}, whose addresses are given in \cite{janus}. 
These contracts currently are not associated with any financial assets, thus posing no harm.
%affect other accounts' financial assets, thus posing no harm.
We find 3 of these contracts that are labeled as phishing contracts by GoPlus Security Lab \cite{goplus}, further reflecting the accuracy of our tool.

We also use Pied-piper and Tokeer on the contracts with \textit{Parameter Manipulation} and \textit{Whitelist} risks, but the tools fail to find these two kinds of centralized risks due to the absence of patterns for them. 
Unlike other approaches, \ourtool successfully detects these risks by identifying differences in balances between privileged and ordinary account executions.
For example, the analysis process of the contract in Fig. \ref{fig:param} by \ourtool is as follows:
\begin{enumerate}[1)]
\item \ourtool executes the functions using the contract owner and an ordinary account, respectively. 
The owner-restricted \texttt{updateFee} causes differing \texttt{txFee} values, so we label this variable in subsequent states.
% Since the \texttt{updateFee} function can only be called by the contract owner, there will be a difference in the \texttt{txFee} variable in the results of the two executions.
% Thus, \texttt{txFee} variable is labeled in the subsequent states.

\item Next, \ourtool identifies functions related to the \texttt{txFee} variable.
Since the execution of the \texttt{\_transfer} function requires reading \texttt{txFee}, \ourtool continues to execute it using different accounts, resulting in two states $s_p$ and $s_o$.

\item  Based on the design of our variable summaries, \ourtool regards \texttt{\_balances[FeeAddress]} as a Numeric Type variable. 
According to Line 11 and 12 of the contract in  Fig. \ref{fig:param}, \texttt{\_balances[FeeAddress]} is related to the labeled variable \texttt{txFee}, thus be labeled and included in the difference set.

\item Since \texttt{\_balances[FeeAddress]} is a financial variable representing the token balance of \texttt{FeeAddress}, our tool determines that there exists a risk.

\end{enumerate}

\subsection{Experiment for our financial variable recognition module}
\label{subsec:comp_var}

To address question Q3, we evaluate our financial variable recognition module on both a \textit{labeled dataset} and the \textit{real-world dataset}. 
 To ensure an accurate evaluation, we use an open-source dataset \cite{opendata} for training our module, avoiding duplication with the real-world dataset. 
 We filter this dataset by removing contracts duplicated in the \textit{real-world dataset} and those without financial variables. We specifically screen two types of contracts: contracts that perform Ether transfers and contracts that perform token transfers.
Using Slither, we identify contracts employing official functions like \texttt{transfer} and \texttt{send} for Ether transfers, as well as ERC-compliant token contracts.
As a supplement, we identify the names of remaining contracts and manually inspected those containing "token" or "coin" in their names.
Ultimately, we obtain 19,986 finance-related contracts.
 
 % We filter this dataset by removing contracts that are duplicates of those in the real-world dataset and those without financial variables, resulting in a \textit{labeled dataset} of 19,986 contracts.

\begin{table}[]
\centering 
 \caption{Effectiveness of our financial variable recognition module using different network structures and edges}
 \label{tab:net}
\begin{tabular}{c|c|c|c}
\hline
\textbf{Network}                   & \textbf{Edges}       & \textbf{Accuracy(\%)}  &  \textbf{F1-score(\%)}    \\ \hline
\multirow{6}{*}{SAGE}     & CFE+DFE     & 96.61  & 96.08 \\
                          & CFE+DFE+RFE & 96.77   & 96.27 \\
                          & CFE+DFE+CDE & 96.85  & 96.39 \\
                          & CFE+DFE+DDE & 96.97  & 96.49 \\
                          & CFE+DFE+FCE & 96.71  & 96.18 \\
                          & All         & 97.47  & 97.04 \\ \hline
\multirow{6}{*}{GCN}      & CFE+DFE     & 97.46  & 96.99 \\
                          & CFE+DFE+RFE & 97.68  & 97.27 \\
                          & CFE+DFE+CDE & 97.53   & 97.14 \\
                          & CFE+DFE+DDE & 98.08   & 97.75 \\
                          & CFE+DFE+FCE & 97.49  & 97.11 \\
                          & All         & 98.22  & 97.90 \\ \hline
\multirow{6}{*}{SAGE+GAT} & CFE+DFE     & 96.58  & 95.98 \\
                          & CFE+DFE+RFE & 97.06  & 96.56 \\
                          & CFE+DFE+CDE & 96.76   & 96.20 \\
                          & CFE+DFE+DDE & 97.19  & 96.69 \\
                          & CFE+DFE+FCE & 97.00   & 96.51 \\
                          & All         & 97.01  & 96.49 \\ \hline
\multirow{6}{*}{GCN+GAT}  & CFE+DFE     & 97.41  & 96.98 \\
                          & CFE+DFE+RFE & 97.66  & 97.31 \\
                          & CFE+DFE+CDE & 97.38  & 96.97 \\
                          & CFE+DFE+DDE & 98.24  & 97.97 \\
                          & CFE+DFE+FCE & 97.38  & 96.99 \\
                          & All         & \textbf{98.64}  & \textbf{98.42} \\ \hline
\end{tabular}
\end{table}

First, we divide the labeled dataset into training, validation, and test sets in a ratio of 8:1:1. 
The training set is used to train graph neural networks, while the module's effectiveness is evaluated on the test set. 
To measure the module's performance in recognizing financial variables, we use two metrics: Accuracy ($\frac{TP+TN}{TP+FP+TN+FN}$) and F1-score ($\frac{2TP}{2TP+FP+FN}$). 
Here, $TP$ represents correctly identified as financial, $FP$ denotes variables incorrectly identified as financial, $FN$ indicates missed financial variables, and $TN$ denotes correctly identified non-financial variables.
In this evaluation, we use various network structures and edge combinations. 
As shown in TABLE \ref{tab:net}, the module using a GCN+GAT structured network, along with all six types of edges described in Section \ref{subsec:financial_var}, achieves the highest Accuracy (98.64\%) and F1 score (98.42\%). 
These results not only demonstrate the effectiveness of our module in recognizing financial variables but also explain our choice of network structure and edge combinations.

Next, we implement JANUS-N, a modified version of our tool that uses a naive method to replace our financial variable recognition module.
This method identifies financial variables based on name similarity to common financial variable names \cite{wang2023automated}.
% replace the existing financial variable recognition module with a naive algorithm and implement a new version of our tool called \ourtool-N.
% The naive algorithm identifies financial variables by comparing the similarity of variable names to commonly used financial variables above a certain threshold.
We compare the effectiveness of JANUS-N and \ourtool in detecting centralized risks on the \textit{real-world dataset}.
The results show that JANUS-N identifies 7,968 smart contracts with centralized risks, 423 fewer than our tool detects.
Upon manual review, we confirm that centralized risks exist in these 423 contracts. 
These contracts use infrequent or obfuscated variable names, which the naive algorithm fails to identify, leading to the underreporting. 
In contrast, our recognition module, which identifies variable features based on contextual information, successfully detects these financial variables.
Note that although our module identifies only 5.31\% (423/7,968) more risks than the naive algorithm, considering that any missed risk could result in financial losses for contract users, using GNN in our module is necessary.

%% file: related.tex
\section{Related Work}
\label{sec:related}

\textbf{Centralized Risks Detection.} There are currently two tools that can automatically detect centralized risks in smart contracts: Pied-Piper and Tokeer. 
Pied-Piper utilizes Datalog analysis to identify common risky code patterns, supplemented by directed fuzzing to reduce false alarms. 
Similarly, Tokeer develops oracles based on known rug pull contract patterns and a model analyzing token contract transfer processes, subsequently employing Datalog analysis to detect rug pull contracts.
%Here, rug pull contracts are a kind of malicious contracts often containing backdoors for asset transfer, hence detection methods designed for these contracts can also identify certain types of backdoors.
However, these tools rely on predefined patterns to detect centralized risks, limiting their ability to identify unknown patterns that may emerge in real-world scenarios.

\textbf{Identification of Malicious Transactions and Accounts.} Some methods focus on identifying malicious smart contract transactions and accounts.
For instance, Cernera et al. \cite{cernera2023token} summarize the transaction features of rug pull contracts to determine if a transaction originates from such a contract. 
Xia et al. \cite{xia2022trade} gather malicious transactions from the Uniswap DEX \cite{uniswap} to train a detection model. 
Hu et al. \cite{hu2021transaction} summarize key features of malicious contract transactions and train an LSTM network \cite{hochreiter1997long} for detecting them. 
Taking a different approach, 
Tan R et al. \cite{tan2021graph} and Zhou J et al. \cite{zhou2022behavior}  employ machine learning to classify the normal and malicious accounts. 
%Their method employs machine learning models to classify accounts as normal, malicious, or belonging to other categories.
% While these methods contribute to identifying malicious transactions and accounts, their goals differ from ours. 
These method focus on maliciousness detection, which may not necessarily involve centralized risks. 
% Consequently, these approaches are unsuitable for the specific identification and localization of centralized risks.
% However, the objectives of these methods differ from ours. 
% These methods focus on identifying malicious transactions and contracts, which may not necessarily contain backdoors. 
% Consequently, they are not suitable for the identification and localization of backdoors in smart contracts.

% The automated detection methods for backdoors and related malicious contracts can be categorized into two main groups: transacti\-on-driven methods and code-driven methods.

% On the one hand, transaction-driven methods analyze transactions associated with smart contracts to determine their malicious nature. 

% However, transaction-driven methods necessitate access to the historical transactions of contracts or simulations of a limited number of transactions. Given the infinite number of transactions that can invoke a smart contract, obtaining exactly those triggering backdoors becomes challenging, potentially leading to the omission of backdoors.

% On the other hand, code-driven methods examine the source code of smart contracts to identify potential backdoors based on predefined patterns. 

\textbf{Vulnerability Detection.} The methods for detecting vulnerabilities in contracts can be mainly categorized into two approaches: those based on code patterns and those that are independent of code patterns.
Methods based on code patterns include ZEUS \cite{DBLP:conf/ndss/KalraGDS18} using symbolic model checking, ContractFuzzer \cite{0001LC18} using fuzzing and SECURIFY \cite{tsankov2018securify} employing Datalog, etc. 
However, their patterns are not designed for centralized risks, making these methods unsuitable for risk detection.
Methods that do not depend on patterns include MAIAN \cite{nikolic2018finding}, eThor \cite{DBLP:conf/ccs/SchneidewindGSM20}, FASVERIF \cite{wang2023automated}, etc. 
MAIAN  finds contracts that violate specific safety or liveness properties of traces.
EThor abstracts bytecode semantics into Horn clauses and expresses properties as reachability queries.
% CFF reasons about the economic security properties of DeFi contracts. 
FASVERIF automatically generates and verifies finance-related properties and models for contracts.
However, these non-pattern-based methods primarily focus on vulnerabilities that external attackers might exploit, overlooking the maliciousness of the contract itself. 
% Consequently, they are not well-suited for risk detection.

\textbf{LLM-based Detection.} The development of Large Language Models (LLMs) has provided new approaches for smart contract security analysis. Existing research has attempted to apply LLMs to smart contract vulnerability detection and has achieved some progress \cite{yu2025smart}\cite{sun2024gptscan}\cite{li2025scalm}. However, in the specific field of smart contract centralization risk detection, no related work introducing LLMs has been conducted to the best of our knowledge.
At the same time, LLMs still exhibit certain inherent limitations. Compared to directly using LLMs for contract analysis, our tool offers advantages in the following aspects: 
\begin{itemize}
    \item The results are deterministic. LLM outputs may vary for the same contract due to randomness, introducing uncertainty into security analysis \cite{chen2025chatgpt}. In contrast, \ourtool provides stable and reproducible results. 
    \item The results are more reliable. LLMs suffer from "hallucination," which may lead to incorrect judgments \cite{huang2025survey}. JANUS, based on symbolic execution with formal proofs, delivers more reliable outcomes and generates reports that pinpoint risky functions, facilitating further manual verification.
\end{itemize}
However, We believe that combining LLMs with \ourtool could enhance overall effectiveness through complementary strengths. For instance, LLMs could assist in identifying financial variables within contracts, improving detection accuracy through semantic understanding. Additionally, LLMs could guide path pruning during symbolic execution \cite{chen2025numscout}, thereby mitigating state explosion. In the future, we will further explore these directions to optimize our tool's performance.

%% file: limitation.tex
%!TEX root = submission.tex

\section{Limitations and Future Work}
\label{sec:limit}

 We summarize several limitations of \ourtool as follows:

\textbf{Limited Detection Scope.} Our tool currently focuses on detecting centralized risks primarily associated with finance-related variables (e.g., token balances) and privileged accounts identified from state variables of the \texttt{address} type. Consequently, our tool may miss non-financial risks (e.g., mutable metadata~\cite{yan2023bad}) or privileged accounts represented using other types, potentially resulting in false negatives. As a supplement, we provide a customizable interface that allows users to specify target variables manually, thereby extending detection to non-financial risks. 
It is noteworthy that even within this limited scope, our tool has successfully identified practical centralized risk cases overlooked by existing tools, demonstrating its effectiveness and value. 
In the future, we plan to integrate large language models, which have been used by multiple studies\cite{sun2024gptscan}\cite{yu2025smart} to understand smart contract semantics, to identify a broader range of key variables and privileged account variables.

\textbf{Dependency on External Static Analysis Engine.} \ourtool leverages Slither for foundational static analysis to avoid reinventing established components. As a result, we cannot definitively determine if the symbolic execution process generates infeasible states or if the data flow analysis leads to omissions in functions that require analysis or variables that need labeling. However, it is important to note that our iterative algorithm is independent and operates on an intermediate representation. This means that while Slither is currently used, our method could be adapted to other analysis engines (e.g., Mythril) that provide a similar code representation. The current implementation's reliability is supported by its performance on the \textit{comparison dataset}, where no obvious issues from Slither's analysis are observed.
In the future, we plan to integrate other static analysis tools and employ cross-validation to enhance the reliability of foundational analysis and develop a generic interface to enable more flexible integration of our core algorithms with different static analysis backends.

%, despite their effectiveness, have inherent limitations. These limitations 

% \item \textbf{Empirical assumption.} In Section 1, we propose Assumption \ref{assume} based on our observations.
% This assumption is empirical and unproven.

\textbf{Symbolization of the results of external calls.} When a contract contains a statement calling an external contract whose code is unknown, our tool represents the result of that external call as a symbolic value. 
While this approach allows for analysis to continue, it may lead to inaccurate detection results. 
This limitation that can only analyze codes given beforehand is an inherent drawback of static analysis methods \cite{onwuzurike2019mamadroid}. 
In the future, we aim to enhance the accuracy of our detection results by combining our tool with dynamic analysis to obtain actual external call results.
%\item 
% \textbf{Conservative strategy for backdoor identification.} As mentioned in Section \ref{subsec:real}, \ourtool identifies potential backdoors, which encompass both intentionally malicious code deployed by attackers and backdoor-like code designed by normal developers due to negligence. 
% Given the challenge of definitively determining a developer's intent, we adopt a conservative strategy: all detected potential backdoors are classified as malicious and capable of causing financial harm to users. 
% This strategy prioritizes user protection, ensuring no economically risky backdoors are missed due to misinterpretation of developer intent.
%\end{itemize}

% \textbf{Limited Detection Scope of Privileged Accounts.} Our tool relies on common code patterns and identifies privileged accounts exclusively from state variables of the \texttt{address} type. When privileged accounts are represented using variables of other types—such as a \texttt{mapping(address => Role)} where \texttt{Role} is a user-defined struct—our tool fails to recognize them, potentially leading to false negatives. To address this limitation, we will incorporate contextual semantic information for more accurate identification of privileged accounts in the future.

\textbf{The trade-off between accuracy and efficiency.} While our tool's execution speed is currently slower than that of pattern-based matching tools, limiting its practical deployment, it offers superior detection accuracy and the capability to identify unknown risks through symbolic execution. These advantages are critical in real-world applications. 
In the future, we plan to further optimize the tool’s efficiency through methods such as parallel analysis of multiple contracts and incorporating preliminary screening with other tools.

%% file: conclusion.tex
%!TEX root = submission.tex

\section{Conclusion}

We propose and implement JANUS, which can detect financial centralized risks independently of their behavior patterns.
JANUS employs a difference-oriented state traversal method and variable summaries to reduce the number of states to be compared.
We prove that these methods do not introduce false alarms or omissions in detection. 
\ourtool outperforms other tools in terms of risk detection accuracy, and it successfully finds two types of risks that other tools fail to detect.
%, \textit{Parameter Manipulation} and \textit{Whitelist}, 

%% file: appendix.tex
\section{Appendix}

\subsection{The Iterative Algorithm for Obtaining Differences in Labeled States}
\label{subsec:appendix1}

To facilitate understanding, a simplified state representation is used in the algorithm in our manuscript. The algorithm we employ in our tool is shown below, which follows the same main flow but replaces the states with labeled states.
In this algorithm, $labeled\_symbolic\_exec(f,a,l)$ symbolically executes the function $f$ using account $a$, and label the variables dependent on the labeled variables in $l$.

\begin{algorithm}
\small
\caption{The Iterative Algorithm for Obtaining Differences in Labeled States}
\label{alg:alg2}
\textbf{Input:}  source code $C$ of a contract \\
\textbf{Output:}  differences bewtween the labeled states reached by executing $C$ with  $a_p$ and  $a_o$
\begin{algorithmic}[1]
\State $F \leftarrow$ the functions in $C$
\State $l_0 \leftarrow$ initial labeled state after executing the constructor of $C$ (The set of labeled variables is $\emptyset$)
\State $D \leftarrow \emptyset$
\State $S_{next} \leftarrow \emptyset$
\For{$f$ in $F$}
\State $l_p \leftarrow labeled\_symbolic\_exec(f,a_p,l_0)$
\State $l_o \leftarrow labeled\_symbolic\_exec(f,a_o,l_0)$
\State $delta \leftarrow \Delta(\Phi(l_p),\Phi(l_o))$
\If{$delta \neq \emptyset$ and $delta \notin D$}
\State $D.add(delta)$
\State $S_{next}.add(l_p)$
\EndIf 
\EndFor
\State $D_{new} \leftarrow D$
\State $l \leftarrow l_0$
\While{$D_{new} \neq \emptyset$}
\State $D_{new}' \leftarrow \emptyset$
\State $S_{next}' \leftarrow \emptyset$
\For{$delta$ in $D_{new}$}
\State $l \leftarrow S_{next}[D_{new}.index(delta)]$
\State $F' \leftarrow related\_funcs\_search(F,delta)$
\State $tmp \leftarrow \emptyset$
\State $S_{tmp} \leftarrow \emptyset$
\For{$f$ in $F'$}
\State $l_p \leftarrow labeled\_symbolic\_exec(f,a_p,l)$
\State $l_o \leftarrow labeled\_symbolic\_exec(f,a_o,l)$
\State $delta' \leftarrow \Delta(\Phi(l_p),\Phi(l_o))$
\If{$delta' \neq \emptyset$ and $delta' \notin tmp$}
\State $tmp.add(delta')$
\State $S_{tmp}.add(l_p)$
\EndIf
\EndFor
\If{$tmp \not\subseteq D$}
\State $D_{new}'.extend(tmp \setminus D)$ 
\State $D.extend(tmp \setminus D)$
\State $S_{next}'.update(S_{tmp},tmp,D)$ 
\EndIf
\EndFor
\State $D_{new} \leftarrow D_{new}'$
\State $S_{next} \leftarrow S_{next}'$
\EndWhile
\end{algorithmic}
\end{algorithm}

\subsection{The Experiments about the Hyperparameters of Our Financial Variable Recognition Module}
\label{subsec:appendix2}

To illustrate the choice of hyperparameters of our recognition module, we conduct experiments for the following hyperparameters:
1) Number of GCN layers. 2) Number of heads in GAT. 3) Combination of input size and hidden size.
Based on the experimental results shown in Figs. \ref{fig:net1}, \ref{fig:net2}, and \ref{fig:net3}, we finally set the number of GCN layers to 5, the number of heads to 2, the input size to 100, and the hidden size to 256, achieving the highest accuracy and F1 score for our module.

\begin{figure}[H]
\centering
\includegraphics[scale=0.33]{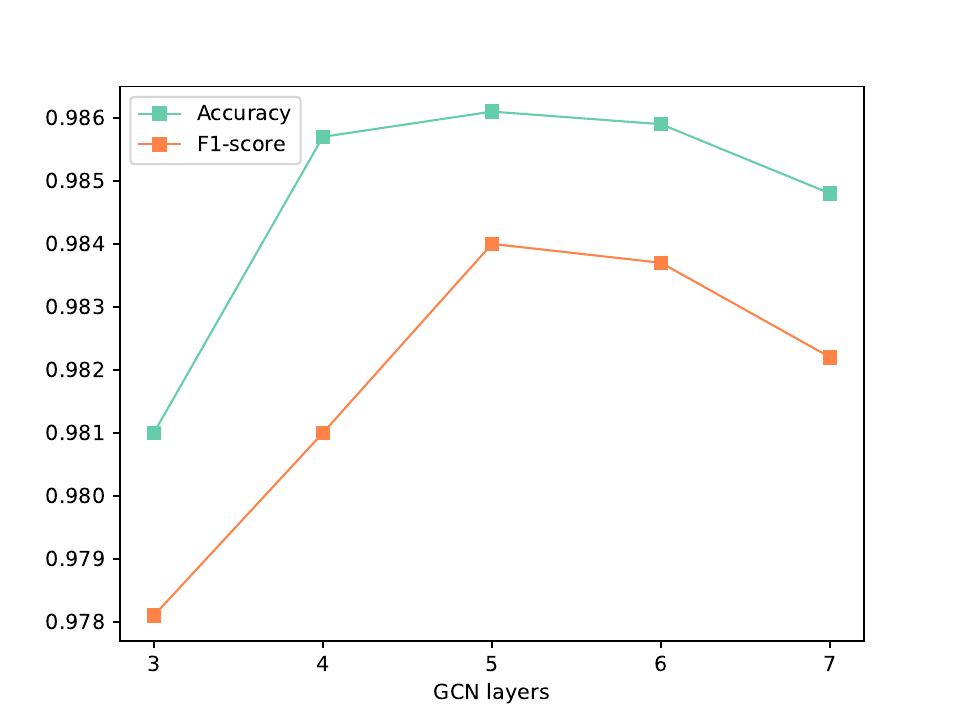}
\caption{Effectiveness of our financial variable recognition module with different number of GCN layers.}
\label{fig:net1}
\end{figure}

\begin{figure}[H]
\centering
\includegraphics[scale=0.33]{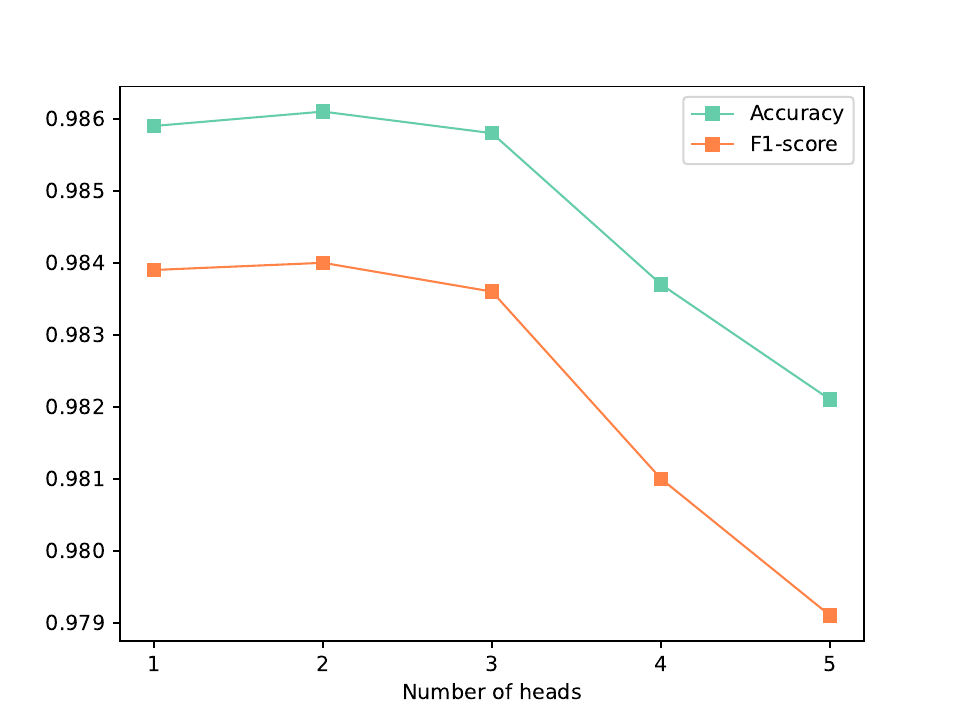}
\caption{Effectiveness of our financial variable recognition module with different number of heads in GAT.}
\label{fig:net2}
\end{figure}

\begin{figure}[H]
\centering
\includegraphics[scale=0.33]{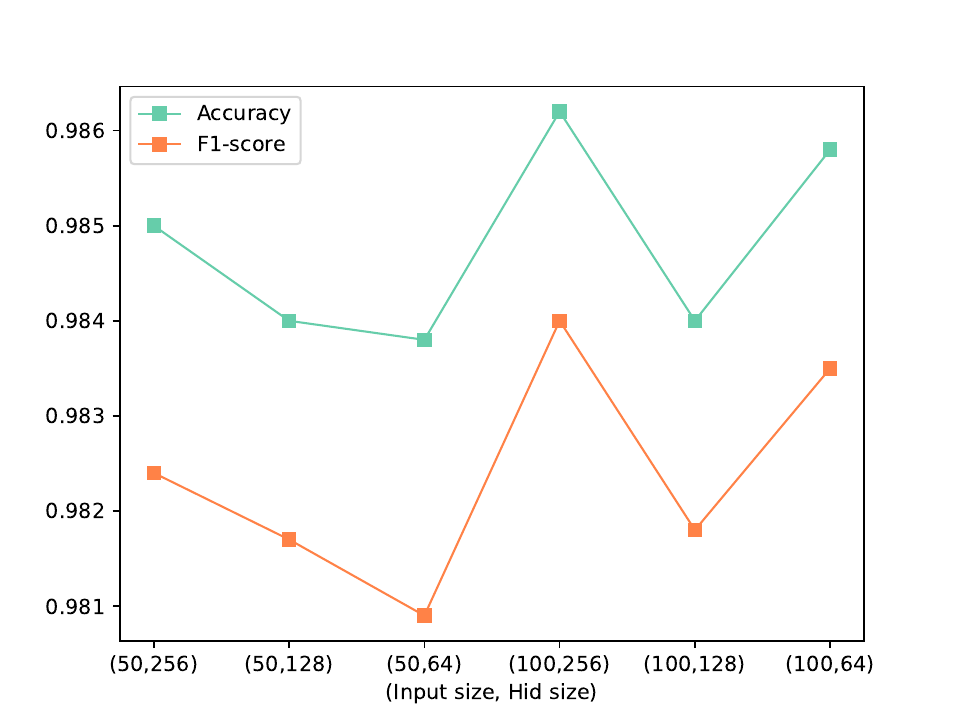}
\caption{Effectiveness of our financial variable recognition module with different input size and hidden size.}
\label{fig:net3}
\end{figure}

\subsection{Theorems and Proofs}

\begin{definition}
\label{def1}
(Differential states) Given a smart contract $C$, if there exist states $s$, $s_p$ and $s_o$ that satisfy 

a) $s \stackrel{f(M)}{\longrightarrow}_{p} s_p$

b) $s \stackrel{f(M)}{\longrightarrow}_{o} s_o$ 

c) $\exists v \in V_C. s_p.\sigma(v) \neq s_o.\sigma(v)$

then we say that $s_p$ and $s_o$ are a pair of differential states of $C$, and $s$ is their source.
The variables that satisfy condition c) are called branch variables of  $s_p$ and $s_o$.
The set of all branch variables for states $s_p$ and $s_o$ is denoted by $BV(s_p,s_o)$.
\end{definition}

\begin{definition}
\label{def2}
(Labeled differential states) Given a smart contract $C$, if there exist labeled states $l$, $l_p$ and $l_o$ that satisfy 

a) $l \stackrel{f(M)}{\longrightarrow}_{p} l_p$

b) $l \stackrel{f(M)}{\longrightarrow}_{o} l_o$

c) $\exists v \in V_C. (l_p.\sigma(v) \neq l_o.\sigma(v) \lor l_p.\theta(v) = true)$

then we say that $l_p$ and $l_o$ are a pair of labeled differential states of $C$, and $l$ is their source.
The variables that satisfy condition c) are called branch variables of  $l_p$ and $l_o$.
The set of all branch variables for labeled states $l_p$ and $l_o$ is denoted by $BV(l_p,l_o)$.
\end{definition}

\begin{table*}[]
 \centering 
 \caption{The Design of Variable Summaries for Different Variable Types}
 \label{tab:sum}
\begin{tabular}{c|c|c|c|c}
\hline
\textbf{Variable Type}                             & \textbf{Key}         & \textbf{Category}        & \textbf{Summarized Value}  & \textbf{Description} \\ \hline
\multirow{3}{*}{Numeric}         & \textit{is\_increased}        & \uppercase\expandafter{\romannumeral 1}  & \{True,False\}    &  Whether the variable's value is increased compared to the last state.           \\ \cline{2-5} 
                             & \textit{is\_descreased}        & \uppercase\expandafter{\romannumeral 1}   & \{True,False\}   &  Whether the variable's value is decreased compared to the last state.           \\ \cline{2-5} 
                             & \textit{related\_const\_var}   & \uppercase\expandafter{\romannumeral 2}  & State variables and constants &  The state variables and constants related to the variable's value.           \\ \hline
\multirow{3}{*}{Address}        & \textit{is\_constant}     & \uppercase\expandafter{\romannumeral 1} & \{True,False\}    &  Whether the variable's value is assigned as a constant address.              \\ \cline{2-5} 
                             & \textit{is\_changed}      & \uppercase\expandafter{\romannumeral 1}  & \{True,False\}   &  Whether the variable's value is changed compared to  the last state.   \\ \cline{2-5}         
                             & \textit{related\_const\_var}  & \uppercase\expandafter{\romannumeral 2} & State variables and constants   &  The state variables and constants  related to the variable's value.           \\ \hline
Mapping/ether                             & \textit{key variable set}          & /   & Summaries of value variables  &  The summaries of value variables corresponding to their types.      \\ \hline
\multirow{3}{*}{exec\_state} & \textit{success}                 & \uppercase\expandafter{\romannumeral 2} & \{True,False\}  & Whether the function is executed successfully.          \\ \cline{2-5} 
                             & \textit{revert}                 & \uppercase\expandafter{\romannumeral 2}  & \{True,False\} & Whether the function is reverted during execution.          \\ \cline{2-5} 
                             & \textit{selfdestruct}            & \uppercase\expandafter{\romannumeral 2}  & \{True,False\}  & Whether the function leads to selfdestruct of the contract.         \\ \hline

\end{tabular}
\begin{tabular}{c}
Category \uppercase\expandafter{\romannumeral 1}: generated based on variable values\ \ \ \ Category \uppercase\expandafter{\romannumeral 2}: generated based on the executed statements \\
\end{tabular}
\end{table*}

% \begin{definition}
% \label{def3}
% For a pair of differential states $s_p$ and $s_o$, we define a branch variable $v$ as one that satisfies the condition $s_p.\sigma(v) \neq s_o.\sigma(v)$. The set of all branch variables for states $s_p$ and $s_o$ is denoted by $BV(s_p,s_o)$.
% \end{definition}

To demonstrate that our state traversal method and variable summaries do not cause false alarms or omissions in detection,
we propose and prove the following theorems (informal) based on the above definitions: 
\begin{enumerate}
\item If a pair of differential states exists in the state space, then a pair of labeled differential states can be found using the difference-oriented state traversal method. (Refer to Theorem \ref{lemma3})
\item  If a pair of labeled differential states is found using the difference-oriented state traversal method, a pair of differential states must exist in the state space. (Refer to Theorem \ref{lemma4})
\item For a pair of labeled differential states, the summaries of all the branch variables are also different. (Refer to Theorem \ref{lemma1})
\end{enumerate}

The formal theorems and their proofs are as follows:

\begin{theorem}
\label{lemma3}

Given a contract $C$, if there exist two transition sequences 
$$s_0 {\stackrel{f_1(M_1)}{\longrightarrow}}_{x_1} s_1 ...  {\stackrel{f_n(M_n)}{\longrightarrow}}_{x_n} s_n $$
$$s_0 {\stackrel{f_1(M_1)}{\longrightarrow}}_{x_1'} s_1' ...  {\stackrel{f_n(M_n)}{\longrightarrow}}_{x_n'} s_n' $$
that satisfy 

a) $s_n$, $s_n'$ are differential states

b) $\exists i.\ x_i \neq x_i'$

then there exist two transition sequences
$$l_0 {\stackrel{f_1(M_1)}{\longrightarrow}}_{y_1} l_1 ...   {\stackrel{f_n(M_n)}{\longrightarrow}}_{p} l_n$$
$$l_0 {\stackrel{f_1(M_1)}{\longrightarrow}}_{y_1'} l_1 ...   {\stackrel{f_n(M_n)}{\longrightarrow}}_{o} l_n'$$
that satisfy

c) $l_0 = (s_0.\sigma,s_0.\pi,\emptyset)$

d) $l_n$, $l_n'$ are labeled differential states and $BV(s_n,s_n') \subseteq BV(l_n,l_n')$

e) $\forall i < n.\ y_i = y_i' = p$
\end{theorem}

\begin{proof}
We proceed by induction over the number of states in one sequence $i$.

\textit{Base case}. For $i$ = 1, assume that $x_1 = p$, then we have that 
$s_0 {\stackrel{f_1(M_1)}{\longrightarrow}}_{p} s_1$, $s_0 {\stackrel{f_1(M_1)}{\longrightarrow}}_{o} s_1'$,
$l_0 {\stackrel{f_1(M_1)}{\longrightarrow}}_{p} l_1$, $l_0 {\stackrel{f_1(M_1)}{\longrightarrow}}_{o} l_1'$.

1) By condition c), we have 
$l_1.\sigma = s_1.\sigma$, $l_1'.\sigma = s_1'.\sigma$.

2) By Definition \ref{def1}, we have 
$\exists v \in V_C.\ s_1.\sigma(v) \neq s_1'.\sigma(v)$

3) By 1) and 2), we have
 $\forall v \in V_C.\ (s_1.\sigma(v) \neq s_1'.\sigma(v)) \leftrightarrow  (l_1.\sigma(v) \neq l_1'.\sigma(v))$

4) By Definition of labeled states, we have
$\forall v \in V_C.\ (l_1.\theta(v) = true) \leftrightarrow  (l_1.\sigma(v) \neq l_1'.\sigma(v))$

5) By 3), 4) and Definition \ref{def2}, we have that $l_1$, $l_1'$ are labeled differential states and $BV(l_1,l_1') = BV(s_1,s_1')$, Theorem \ref{lemma3} holds. (The proofs of the case that $x_1 = o$ is similar and we omit it for brevity.)

\textit{Inductive step}. Assume Theorem \ref{lemma3} holds for $i \geq 1$. We show that the theorem holds for $i + 1$ case by case.

\textbf{Case 1}. $x_{i+1} = x_{i+1}'$.

1) For all $v \in ( BV(s_{i+1},s_{i+1}') \setminus BV(s_i,s_i') )$, $v$ is dependent of variables in $BV(s_i,s_i')$.

2) By the assumption that Theorem \ref{lemma3} holds for $i$ and Definition of labeled states, we have $\forall v \in BV(s_i,s_i').\ l_{i+1}.\theta(v) = true$.

3) By 1) and 2), we have $\forall v \in ( BV(s_{i+1},s_{i+1}') \setminus BV(s_i,s_i') ).\ l_{i+1}.\theta(v) = true$.

4) By 2) and 3), we have $\forall v \in BV(s_{i+1},s_{i+1}').$ $l_{i+1}.\theta(v) = true$.

5) By 4) and Definition \ref{def2}, we have that $l_{i+1}$, $l_{i+1}'$ are labeled differential states and $BV(s_{i+1},s_{i+1}') \subseteq BV(l_{i+1},l_{i+1}')$,  Theorem \ref{lemma3} holds for $i+1$.

\textbf{Case 2}. $x_{i+1} \neq x_{i+1}'$.

1) We have $s_i {\stackrel{f_{i+1}(M_{i+1})}{\longrightarrow}}_{p} s_{i+1}$, $s_i {\stackrel{f_{i+1}(M_{i+1})}{\longrightarrow}}_{o} s_{i+1}'$,
   $l_i {\stackrel{f_{i+1}(M_{i+1})}{\longrightarrow}}_{p} l_{i+1}$, $l_i {\stackrel{f_{i+1}(M_{i+1})}{\longrightarrow}}_{o} l_{i+1}'$.

2) By Definition \ref{def1}, we have $\forall v \in ( BV(s_{i+1},s_{i+1}') \setminus BV(s_i,s_i') ). s_{i+1}.\sigma(v) \neq s_{i+1}'.\sigma(v)$.

3) For $v \in ( BV(s_{i+1},s_{i+1}') \setminus BV(s_i,s_i') )$, we consider two cases. If $v$ is dependent on variables in $BV(s_i,s_i')$, the proofs are similar to Case 1 and are omitted. For variables $v$ that are not dependent on those in $BV(s_i,s_i')$, by 1) and 2), we  have $l_{i+1}.\sigma(v) \neq l_{i+1}'.\sigma(v)$.

4) By 3), we have $\forall v \in ( BV(s_{i+1},s_{i+1}') \setminus BV(s_i,s_i') ). v \in BV(l_{i+1}, l_{i+1}')$.

5) By the assumption that Theorem \ref{lemma3} holds for $i$ and Definition of labeled states, we have $\forall v \in BV(s_i,s_i').\ l_{i+1}.\theta(v) = true$.

6) By 4) and 5), we have that $l_{i+1}$, $l_{i+1}'$ are labeled differential states and $BV(s_{i+1},s_{i+1}') \subseteq BV(l_{i+1},l_{i+1}')$,  Theorem \ref{lemma3} holds for $i+1$.

\end{proof}

\begin{theorem}
\label{lemma4}

Given a contract $C$, if there exist two labeled transition sequences
$$l_0 {\stackrel{f_1(M_1)}{\longrightarrow}}_{y_1} l_1 ...   {\stackrel{f_n(M_n)}{\longrightarrow}}_{p} l_n$$
$$l_0 {\stackrel{f_1(M_1)}{\longrightarrow}}_{y_1'} l_1 ...   {\stackrel{f_n(M_n)}{\longrightarrow}}_{o} l_n'$$
that satisfy

a) $l_0.\theta = \emptyset$

b) $l_n$, $l_n'$ are labeled differential states 

c) $\forall i < n.\ y_i = y_i' = p$

then there exist two transition sequences 
$$s_0 {\stackrel{f_1(M_1)}{\longrightarrow}}_{x_1} s_1 ...  {\stackrel{f_n(M_n)}{\longrightarrow}}_{x_n} s_n $$
$$s_0 {\stackrel{f_1(M_1)}{\longrightarrow}}_{x_1'} s_1' ...  {\stackrel{f_n(M_n)}{\longrightarrow}}_{x_n'} s_n' $$
that satisfy 

d) $s_0 = (l_0.\sigma,l_0.\pi)$

e) $s_n$, $s_n'$ are differential states and $BV(l_n,l_n')  \subseteq BV(s_n,s_n')$

f) $\exists i.\ x_i \neq x_i'$

\end{theorem}

\begin{proof}
We proceed by induction over the number of states in one sequence $i$.

\textit{Base case}. Similar to \textit{Base case} of proofs for Theorem \ref{lemma3}.

\textit{Inductive step}. Assume Theorem \ref{lemma4} holds for $i \geq 1$. %We show that the theorem holds for $i + 1$ case by case.
By condition b) and Defintion of labeled differential states , we have $\forall v \in BV(l_{i+1},l_{i+1}').\ l_{i+1}.\sigma(v) \neq l_{i+1}'.\sigma(v) \lor l_{i+1}.\theta(v) = true$. For each $v \in  BV(l_{i+1},l_{i+1}')$, we discuss whether it appears in  $BV(s_n,s_n')$ case by case.

\textbf{Case 1}. $l_{i+1}.\theta(v) = true$.

1) By the assumption that Theorem \ref{lemma4} holds for $i$, there exists a transition $l_{i-1} {\stackrel{f_i(M_i)}{\longrightarrow}}_{o} l_o$. $l_i$, $l_o$ are labeled differential states and $v \in BV(l_i,l_o)$.

2) By 1), we can construct two transition sequences 
$$s_0 {\stackrel{f_1(M_1)}{\longrightarrow}}_{x_1} s_1 ... {\stackrel{f_i(M_i)}{\longrightarrow}}_{p} s_i$$
$$s_0 {\stackrel{f_1(M_1)}{\longrightarrow}}_{x_1'} s_1' ... {\stackrel{f_i(M_i)}{\longrightarrow}}_{o} s_i' $$
such that 
$\forall j < i.\ x_j = x_j' = p$ and $s_i.\sigma = l_i.\sigma$ and $s_i'.\sigma = l_o.\sigma$.

3) By 1) and 2), $s_i.\sigma(v) \neq s_i'.\sigma(v)$.

4) By 2) and 3),  we can construct two transition sequences 
$$s_0 {\stackrel{f_1(M_1)}{\longrightarrow}}_{x_1} s_1 ... {\stackrel{f_i(M_i)}{\longrightarrow}}_{p} s_i {\stackrel{f_{i+1}(M_{i+1})}{\longrightarrow}}_{p} s_{i+1}$$
$$s_0 {\stackrel{f_1(M_1)}{\longrightarrow}}_{x_1'} s_1' ... {\stackrel{f_i(M_i)}{\longrightarrow}}_{o} s_i' {\stackrel{f_{i+1}(M_{i+1})}{\longrightarrow}}_{o} s_{i+1}' $$
such that 
$\forall j < i.\ x_j = x_j' = p$ and $s_{i+1}.\sigma(v) \neq s_{i+1}'.\sigma(v)$, thus $v \in BV(s_{i+1},s_{i+1}')$.

\textbf{Case 2}. $l_{i+1}.\sigma(v) \neq l_{i+1}'.\sigma(v)$.

Since the source of $l_{i+1}$ and $l_{i+1}$ are same, the value of $v$ can be assigned as different by executing $f_{i+1}$ using different accounts.
Therefore, we can construct two transition sequences 
$$s_0 {\stackrel{f_1(M_1)}{\longrightarrow}}_{x_1} s_1 ...  {\stackrel{f_{i+1}(M_{i+1})}{\longrightarrow}}_{p} s_{i+1}$$
$$s_0 {\stackrel{f_1(M_1)}{\longrightarrow}}_{x_1'} s_1' ... {\stackrel{f_{i+1}(M_{i+1})}{\longrightarrow}}_{o} s_{i+1}' $$
such that 
$s_{i+1}.\sigma(v) \neq s_{i+1}'.\sigma(v)$, thus $v \in BV(l_i,l_o)$.

Based on the above two cases, we have $\forall v \in  BV(l_{i+1},l_{i+1}'). v \in BV(s_{i+1},s_{i+1}')$. Theorem \ref{lemma4} holds for $i+1$.

\end{proof}

\begin{theorem}
\label{lemma1}

Given a contract $C$, if there exist two labeled transition sequences
$$l_0 {\stackrel{f_1(M_1)}{\longrightarrow}}_{y_1} l_1 ...   {\stackrel{f_n(M_n)}{\longrightarrow}}_{p} l_n$$
$$l_0 {\stackrel{f_1(M_1)}{\longrightarrow}}_{y_1'} l_1 ...   {\stackrel{f_n(M_n)}{\longrightarrow}}_{o} l_n'$$
that satisfy

a) $l_0.\theta = \emptyset$

b) $l_n$, $l_n'$ are labeled differential states 

c) $\forall i < n.\ y_i = y_i' = p$

then we have 

$\forall v \in BV(l_n,l_n').\ (\phi(l_{n-1},l_n,v) \neq \phi(l_{n-1},l_n',v) \lor l_n.\theta(v) = true)$.
\end{theorem}

\begin{proof}
% In the following, we prove the above lemma according to the type of variable $v$ in $BV(s_p,s_o)$ .
By Definition \ref{def2}, we have $\forall v \in BV(l_n,l_n'). l_n.\sigma(v) \neq l_n'.\sigma(v) \lor l_n.\theta(v) = true$. 
For variables that satisfy $l_n.\theta(v) = true$, Theorem \ref{lemma1} obviously holds.
So we dicuss whether Theorem \ref{lemma1} holds for the variable $v$ that satisfies $l_n.\theta(v) = false$:

\textbf{Case 1}. $ l_n.\sigma(v) = l_{n-1}.\sigma(v) \neq l_n'.\sigma(v)$  or $ l_n.\sigma(v) \neq l_{n-1}.\sigma(v) = l_n'.\sigma(v)$.

1) By Definition of $\phi_{\sigma}$, we have $\phi_{\sigma}(l_{n-1}.\sigma,l_n.\sigma,v) \neq \phi_{\sigma}(l_{n-1}.\sigma,l_n'.\sigma,v)$.

2) By 1) and Definition of $\phi$, we have $\phi(l_{n-1},l_n,v) \neq \phi(l_{n-1},l_n',v)$.

\textbf{Case 2}. $ l_n.\sigma(v) \neq l_{n-1}.\sigma(v) \neq l_n'.\sigma(v)$.

\textbf{Subcase 2a}.  The variable $v$ is of Numeric Type.

1) If the growth trends of $l_n.\sigma(v)$ and $l_n'.\sigma(v)$ compared to $l_{n-1}.\sigma(v)$ are different, $\phi_{\sigma}(l_{n-1}.\sigma,l_n.\sigma,v) \neq \phi_{\sigma}(l_{n-1}.\sigma,l_n'.\sigma,v)$. Theorem \ref{lemma1} holds.

2) Otherwise, assume that 
 $l_n.\sigma(v) = l_{n-1}.\sigma(v)+x$ and $l_n'.\sigma(v) = l_{n-1}.\sigma(v)+y$ 
or 
$l_n.\sigma(v) = l_{n-1}.\sigma(v)-x$ and $l_n'.\sigma(v) = l_{n-1}.\sigma(v)-y$,
then we have that

3) By 2), we have $x \neq y$.

4) By 3), we get that $l_n.\pi \neq l_n'.\pi$ since they contain different assignment statements for $v$. 

%The components of $x$ and $y$ fall into three categories: the values of function parameters, the values of state variables, and constants.
5) By 4) and TABLE \ref{tab:sum}, we have 
\begin{align}
\phi(l_{n-1},l_n,v)[related\_const\_var] \neq \notag \\ \phi(l_{n-1},l_n',v)[related\_const\_var]\notag
\end{align}
Therefore, $\phi(l_{n-1},l_n,v) \neq \phi(l_{n-1},l_n',v)$. Theorem \ref{lemma1} holds.

\textbf{Subcase 2b}. The variable $v$ is of Address Type.

1) If one of $l_n.\sigma(v)$ and $l_n'.\sigma(v)$ is a constant address and the other is related to a variable whose value can be changed, we have that
$$\phi(l_{n-1},l_n,v)[is\_constant] \neq \phi(l_{n-1},l_n',v)[is\_constant]$$
Theorem \ref{lemma1} holds.

2) Otherwise, assume that  $l_n.\sigma(v) = x$ and $l_n'.\sigma(v) = y$.

3) By 2), we have $x \neq y$. 

4) By 3), we get that $l_n.\pi \neq l_n'.\pi$ since they contain different assignment statements for $v$. 

5) By 4) and TABLE \ref{tab:sum}, we have 
\begin{align}
\phi(l_{n-1},l_n,v)[related\_const\_var] \neq \notag \\ \phi(l_{n-1},l_n',v)[related\_const\_var]\notag
\end{align}
Therefore, $\phi(l_{n-1},l_n,v) \neq \phi(l_{n-1},l_n',v)$. Theorem \ref{lemma1} holds.

% \textbf{Subcase 2c}. The variable $v$ is of Bool Type.

% Since the value of a Bool variable can only be $true$ or $false$, when $l_n.\sigma(v) \neq l_n'.\sigma(v)$, $ l_n.\sigma(v) \neq \sigma(v) \neq l_n'.\sigma(v)$ never holds.
% So the variable  $v$  cannot be of Bool Type.

\textbf{Subcase 2c}.  The variable $v$ is \texttt{exec\_state}.

By Definition of $\texttt{exec\_state}$ and  TABLE \ref{tab:sum}, the value of $\texttt{exec\_state}$ is corresponding to the values of keys in its summary.
Therefore, if $l_n.\sigma(v) \neq l_n'.\sigma(v)$, we have $\phi(l_{n-1},l_n,v) \neq \phi(l_{n-1},l_n',v)$. Theorem \ref{lemma1} holds.
%$l_n.\sigma(v) \neq l_n'.\sigma(v)$ means that the execution state of $l_n$ and $l_n'$ are different.
% Then, according to TABLE \ref{tab:sum}, one of the following conditions holds: 
% $$\phi(l_{n-1},l_n,v)[success] \neq \phi(l_{n-1},l_n',v)[success]$$
% $$\phi(l_{n-1},l_n,v)[revert] \neq \phi(l_{n-1},l_n',v)[revert]$$
% $$\phi(l_{n-1},l_n,v)[selfdestruct] \neq \phi(l_{n-1},l_n',v)[selfdestruct]$$
%Therefore, $\phi(l_{n-1},l_n,v) \neq \phi(l_{n-1},l_n',v)$. Theorem \ref{lemma1} holds.

\textbf{Subcase 2d}. The variable $v$ is of Mapping Type or variable $v$ is \texttt{ether}.

1) since  $l_n.\sigma(v)$ and $l_n'.\sigma(v)$ are maps, we assume that the key variable set of $l_n.\sigma(v)$ is $K_p$ and the key variable set of $l_n'.\sigma(v)$ is $K_o$.

2) By TABLE \ref{tab:sum}, the keys of $\phi(l,l_n,v)$ is $K_p$ and the keys of $\phi(l,l_n',v)$ is $K_o$.

3) If $K_p \neq K_o$, By 2), we have  $\phi(l,l_n,v) \neq \phi(l,l_n',v)$. Theorem \ref{lemma1} holds.

4) Otherwise, if $K_p = K_o$, according to $K_p = K_o$ and $l_n.\sigma(v) \neq l_n'.\sigma(v)$, we have that $$\exists k \in K_p. l_n.\sigma(v[k]) \neq l_n'.\sigma(v[k])$$
Then we can prove $\phi(v,l_n) \neq \phi(v,l_n')$ according to the type of $v[k]$, which is similar to the above cases, and we omit this part of proofs for brevity.
\end{proof}